%% file: main.tex
\documentclass[lettersize,journal]{IEEEtran}

\input{0_package}
\newtheorem{proposition}{Proposition}
\begin{document}

\title{Uncertainty-Aware Graph Neural Networks: A Multi-Hop Evidence Fusion Approach}

\author{Qingfeng Chen, Shiyuan Li, Yixin Liu, Shirui Pan,~\IEEEmembership{Senior Member,~IEEE}, Geoffrey I. Webb,~\IEEEmembership{Fellow,~IEEE}, Shichao Zhang,~\IEEEmembership{Senior Member,~IEEE}%
\thanks{Q. Chen and S. Li contributed equally to this work.}%
\thanks{This work was partially supported by the Specific Research Project of Guangxi for Research Bases and Talents GuiKe AD24010011, and the Innovation Project of Guangxi Graduate Education YCSW2024138.}
\thanks{Manuscript received April 19, 2021; revised August 16, 2021.}}

\markboth{Journal of \LaTeX\ Class Files,~Vol.~14, No.~8, August~2021}%
{Shell \MakeLowercase{\textit{et al.}}: Uncertainty-Aware Graph Neural Networks: A Multi-Hop Evidence Fusion Approach}


\maketitle

\begin{abstract}
\input{0_Abstract}
\end{abstract}


\section{Introduction}\label{sec:introduction}
\input{1_Introduction}

\section{Related Work}
\input{6_RW}

\section{Preliminaries}
\input{2_Pre}

\section{Motivation and Analysis}\label{DESIGN}
\input{3_Moti}

\section{Methodology}
\input{4_Method}

\section{Experiments}
\input{5_Exp}

\section{Conclusion}
\input{7_Conclusion}


\bibliographystyle{IEEEtran}
\bibliography{bibtex/IEEEexample}

\end{document}

%% file: 0_package.tex
\usepackage[dvipsnames]{xcolor}
\usepackage{hyperref}
\usepackage{booktabs}
\usepackage{mathrsfs}
\usepackage{dutchcal}
\usepackage{xspace}
\usepackage{amsthm}
\usepackage{amsmath,amssymb,amsfonts}
\usepackage{algorithmic}
\usepackage{algorithm}
\usepackage{array}

\usepackage[caption=false,font=footnotesize,labelfont=rm,textfont=rm]{subfig}

\usepackage{textcomp}
\usepackage{stfloats}
\usepackage{url}
\usepackage{verbatim}
\usepackage{graphicx}
\usepackage{cite}

\newcommand{\ourmethod}{EFGNN\xspace}
\definecolor{fgreen}{RGB}{177,207,149}
\definecolor{fred}{RGB}{234,179,138}

\definecolor{firstcolor}{RGB}{20,128,85}
\definecolor{secondcolor}{RGB}{20,104,168}
\definecolor{thirdcolor}{RGB}{236,84,20}

\newcommand{\first}[1]{\textcolor{firstcolor}{\underline{\mathbf{#1}}}}
\newcommand{\second}[1]{\textcolor{secondcolor}{\underline{#1}}}
\newcommand{\third}[1]{\textcolor{thirdcolor}{\underline{#1}}}

%% file: 0_Abstract.tex
Graph neural networks (GNNs) excel in graph representation learning by integrating graph structure and node features. Existing GNNs, unfortunately, fail to account for the uncertainty of class probabilities that vary with the depth of the model, leading to unreliable and risky predictions in real-world scenarios. To bridge the gap, in this paper, we propose a novel \underline{\textbf{E}}vidence \underline{\textbf{F}}using \underline{\textbf{G}}raph \underline{\textbf{N}}eural \underline{\textbf{N}}etwork (\ourmethod for short) to achieve trustworthy prediction, enhance node classification accuracy, and make explicit the risk of wrong predictions. In particular, we integrate the evidence theory with multi-hop propagation-based GNN architecture to quantify the prediction uncertainty of each node with the consideration of multiple receptive fields. Moreover, a parameter-free cumulative belief fusion (CBF) mechanism is developed to leverage the changes in prediction uncertainty and fuse the evidence to improve the trustworthiness of the final prediction. To effectively optimize the \ourmethod model, we carefully design a joint learning objective composed of evidence cross-entropy, dissonance coefficient, and false confident penalty. The experimental results on various datasets and theoretical analyses demonstrate the effectiveness of the proposed model in terms of accuracy and trustworthiness, as well as its robustness to potential attacks. The source code of \ourmethod is available at \href{https://github.com/Shiy-Li/EFGNN}{https://github.com/Shiy-Li/EFGNN}.

%% file: 1_Introduction.tex
\IEEEPARstart{R}{eal-world} data is characterized by uncertainty, which stems from its incompleteness, inconsistency, and noise. Graphs, as a structure representing non-Euclidean data, naturally inherit this uncertainty\cite{akcora2019graphboot}. To handle learning problems on graph data, graph neural networks (GNNs) are efficient message passing-based neural network models that excel in numerous domains, including recommendation systems, drug discovery, and protein design~\cite{velivckovic2017graph,liu2022graph,bronstein2021geometric,baek2021accurate,seo2022siren,bai2024haqjsk}. However, the majority of GNNs often neglect the uncertainty inherent in the data, which hampers their applicability in safety-critical domains~\cite{amodei2016concrete,zhang2022trustworthy}. 
For instance, in a co-purchase network, unreliable predictions by the model may lead to incorrect recommendations. This could result in customers missing out on relevant products or being recommended inappropriate items, ultimately affecting customer satisfaction and sales performance.
This inspired us to develop novel graph learning models to generate sufficiently trustworthy decisions.

\begin{figure}[tbp]
\centerline{\includegraphics[width=\columnwidth]{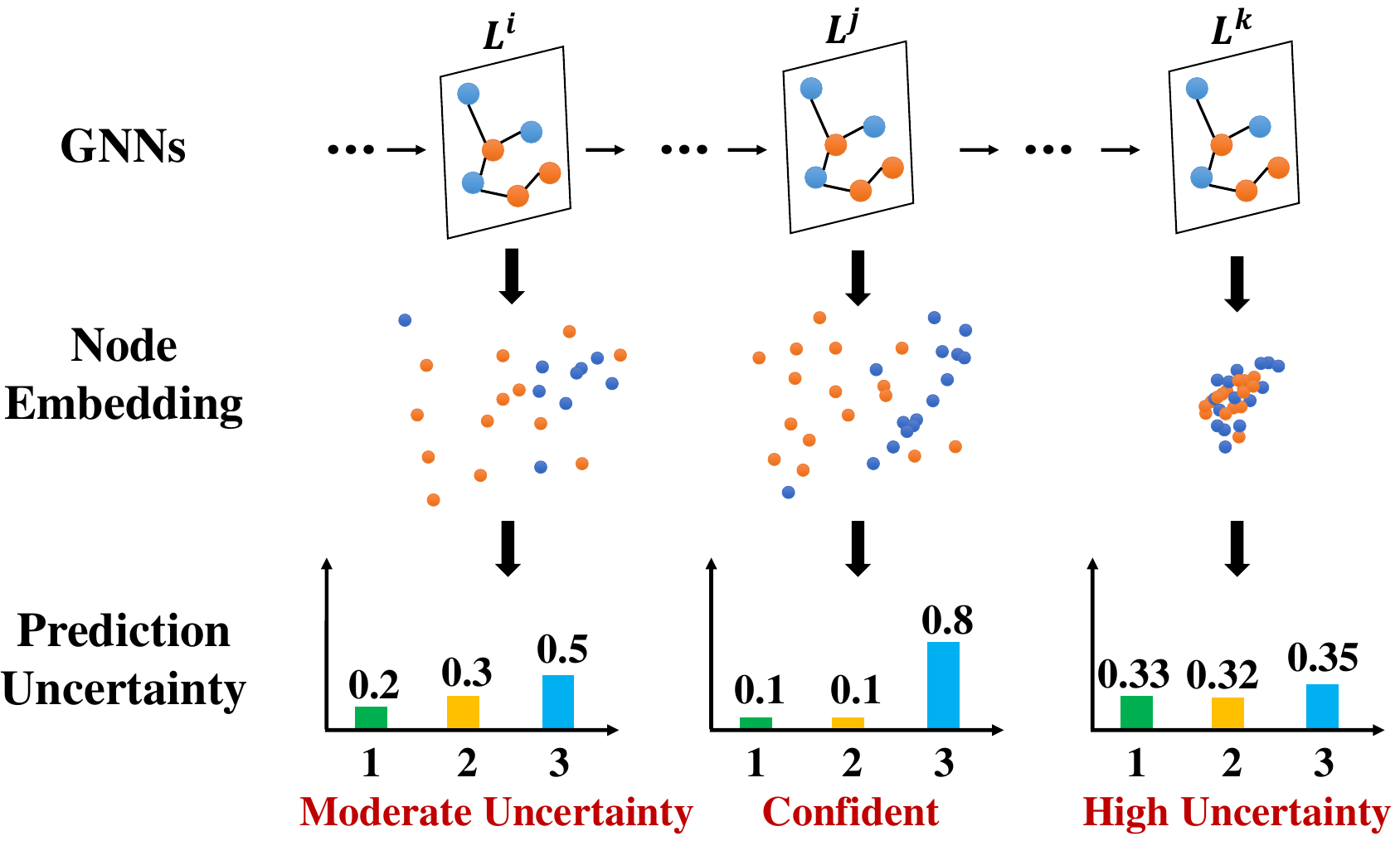}}
\caption{Node embeddings from different hops of neighborhoods may exhibit varying levels of prediction uncertainty.}
\label{fig:sim_intro}
\end{figure}

A trustworthy decision requires the model to not only give the decision outcome but also answer the question, ``How reliable is this decision?'' To this end, the model needs to provide an accurate confidence level for the prediction, i.e., prediction uncertainty~\cite{guo2017calibration}, to enhance the reliability of complex machine learning systems~\cite{charpentier2020posterior}. 
To expose the uncertainty of predictions, existing Bayesian methods~\cite{malinin2018predictive,zhang2019bayesian,neal2012bayesian} partition uncertainty into aleatoric uncertainty (AU) and epistemic uncertainty (EU), and seek to quantify uncertainty by establishing distributions of learnable weight parameters. Meanwhile, subjective logic (SL)-based methods~\cite{han2022trusted, sensoy2018evidential, han2021trusted} concentrate more on modeling the degree of trust in an event or proposition and how to reason based on evidence, viewing vacuity (lack of evidence) as uncertainty in subjective opinions. 
To establish trustworthy graph learning, a series of approaches, termed uncertainty-aware GNNs~\cite{zhao2020uncertainty, wang2021confident, liu2020uncertainty, stadler2021graph}, aim to integrate uncertainty into the architecture of GNNs.
For example, Zhao et al. first revealed the intrinsic relationship in GNNs between various uncertainties, including AU, EU, and vacuity~\cite{zhao2020uncertainty}. CaGCN~\cite{wang2021confident}, differently, employs a confidence calibration strategy with topological perception to solve the low confidence problem in GNNs.

Despite their improvement, the existing uncertainty-aware GNN models only assess the prediction uncertainty of the output layer of shallow GNNs, neglecting its variations with model depth, which can amplify misclassification risks. As shown in Figure \ref{fig:sim_intro}, when multiple GNN layers are stacked to mine higher-order structure information, the quality of node embeddings changes, consequently affecting the prediction uncertainty. The nonnegligible impact of model depth on uncertainty, unexpectedly, can lead to high-risk erroneous predictions. To gain an insightful understanding and quantify this impact, a natural research question arises: \textbf{\textit{How does prediction uncertainty change at different depths of GNNs?}}
 
To answer this question, we conduct a series of empirical analyses to assess the distribution of confidence for nodes of various classes at varying depths (as detailed in Section \ref{DESIGN}). Our empirical research reveals that the impact of changes in model depth on the confidence associated with different class labels is inconsistent and, in some cases, diametrically opposed. Furthermore, the optimal model depth for achieving the highest confidence varies across different classes. Therefore, to acquire confident predictions for each node, it is of great necessity to provide complete neighborhood information via deepening GNN models. 
Nevertheless, our experiments also highlight that with increased model depth, there is a greater proportion of high-risk predictions, which in turn leads to performance degradation. 
In this case, a follow-up question is: \textbf{\textit{How can we further leverage this change in uncertainty to generate more accuracy predictions?}}

To handle the aforementioned research question, in this paper, we propose a novel uncertainty-aware GNN termed \ourmethod. Our theme is to learn confident predictions by exploiting and fusing uncertainty results across different depths of the GNN model. More specifically, to provide predictions with uncertainty for each node considering different receptive fields, we devise a multi-hop evidence graph learning module. The proposed module generates the corresponding evidence based on information subsets of nodes in each neighborhood and derives prediction results with uncertainty using subjective logic theory. Meanwhile, recognizing that the optimal depth varies for different nodes, we introduce a cumulative belief fusion (CBF) mechanism that leverages all available evidence to produce the most trustworthy prediction. Finally, we design a loss function incorporating three loss terms, including evidence cross-entropy loss, dissonance coefficient loss, and Kullback-Leibler (KL) divergence loss, to effectively optimize \ourmethod. We delve into theoretical discussions and conduct empirical experiments to validate the efficacy of the proposed approach.
In summary, the contributions of this paper are: 

\begin{itemize}
    \item \textbf{Finding.} We identify the impact of model depth on the uncertainty predicted by GNNs, and provide a comprehensive analysis to investigate the patterns of the impact.
    \item \textbf{Method.} We propose a novel uncertainty-aware GNN, \ourmethod, that exploits the uncertainty results across different depths to generate trustworthy predictions via multi-hop evidence learning and CBF-based evidence fusion. 
    \item \textbf{Experiments.} We conduct theoretical analysis and extensive experiments to verify the effectiveness of \ourmethod. Experimental results on multiple datasets prove the trustworthiness and robustness of the proposed model.
\end{itemize}

%% file: 6_RW.tex
\subsection{Graph Neural Networks}
\noindent Graph neural networks (GNNs) have emerged as a solution to the limitations of traditional deep neural networks in handling non-Euclidean data spaces and have gained widespread use for capturing relationships or interactions in data~\cite{liu2021combining,monti2017geometric,gilmer2017neural,wang2024unifying}. They work by iteratively aggregating diverse neighborhood information to the target node through two fundamental operations: embedding propagation (EP) and embedding transformation (ET). One of the most notable GNN models is  the graph convolutional network (GCN), which employs Laplacian-symmetric normalized adjacency matrices for spectral-domain convolution operations, delivering outstanding results in node classification tasks~\cite{kipf2016semi}. In addition, other popular GNN models, such as GraphSAGE~\cite{hamilton2017inductive} and GAT~\cite{velivckovic2017graph}, have designed specific EP (i.e., sampling and attention) to improve the model. Except the approach that couples EP and ET operations, another type of GNNs has decoupled architecture where two operations are executed separately~\cite{liu2020towards,liu2024arc,li2024noise}. In these GNN models, EP/ET is performed multiple times before switching to the other operation, as seen in models like APPNP~\cite{gasteiger2018predict}, AP-GCN~\cite{spinelli2020adaptive}, SGC~\cite{wu2019simplifying}, and SIGN~\cite{frasca2020sign}. 

Recent studies have also addressed the challenges of oversmoothing and robustness in GNNs. For instance, Luo et al.~\cite{luo2021learning} propose an adaptive edge dropping mechanism guided by topological denoising, which effectively alleviates the adverse effects of noisy or redundant graph connections. Similarly, Zheng et al.~\cite{zheng2020robust} tackle the oversmoothing problem by employing neural sparsification to selectively prune less informative connections, thus enhancing the discriminative power and robustness of the learned representations.
On the other hand, Huang et al. proposed GCN-RW~\cite{huang2022graph}, enhancing training efficiency via random filters and a regularized least squares loss, which delivers accuracy comparable to state-of-the-art methods with reduced computational costs.
In addition, GNNs have been applied in a wide range of applications, including but not limited to neuromorphic vision, industrial fault diagnosis, and knowledge management~\cite{zheng2023auto,alkendi2022neuromorphic,chen2021interaction,li2024guest,liu2024self,wang2024goodat,pan2025label}.

\subsection{Uncertainty-aware Neural Networks}
\noindent In order to build more trustworthy AI models, a growing number of researchers are working to understand and quantify the uncertainty in neural network predictions. Depending on the different types and sources of uncertainty, various methods have been proposed to measure and quantify uncertainty in neural networks~\cite{tan2020explainable,gawlikowski2023survey}, such as Bayesian neural networks (BNNs)~\cite{posch2020correlated,gal2016dropout,denker1990transforming,kristiadi2022being}, and subjective logic (SL)~\cite{han2022trusted,sensoy2018evidential,han2021trusted}. 

In the field of graph machine learning, there is a branch of methods termed uncertainty-aware GNNs that aims to measure and reduce the uncertainty of GNN models. Among them, BGCNN is a representative method based on Bayesian graph convolutional neural networks for semi-supervised classification. This method realizes the effective utilization of a small amount of label data by inferring the joint posteriori of random graph parameters and node (or graph) labels~\cite{zhang2019bayesian}. Meanwhile, UaGGP is a graph Gaussian process method based on uncertainty perception, which utilizes prediction uncertainty and label smoothing regularization to guide mutual learning\cite{liu2020uncertainty}. In addition, Stadler et al. fully discuss the uncertainty estimation of non-independent node-level prediction, and propose GPN that can clearly perform Bayesian posterior update on the prediction of interdependent nodes~\cite{stadler2021graph}. In contrast to BNNs that infer prediction uncertainty indirectly through weight uncertainty, SL can express the uncertainty of subjective opinion in a simple and intuitive way. Then, Zhao et al. first provided the theoretical proof to explain the relationship between different types (i.e., Bayesian and SL) of uncertainty~\cite{zhao2020uncertainty}.

Although uncertainty-aware GNNs have achieved success, prior research has often overlooked the variations in prediction uncertainty at different depths. In contrast to previous approaches, we directly assess the prediction uncertainty of nodes in diverse neighborhoods using evidence theory and employ a cumulative belief fusion operator with theoretical assurance to generate the most trustworthy predictions.

%% file: 2_Pre.tex
\begin{figure}[tbp]
\centerline{\includegraphics[width=0.5\textwidth]{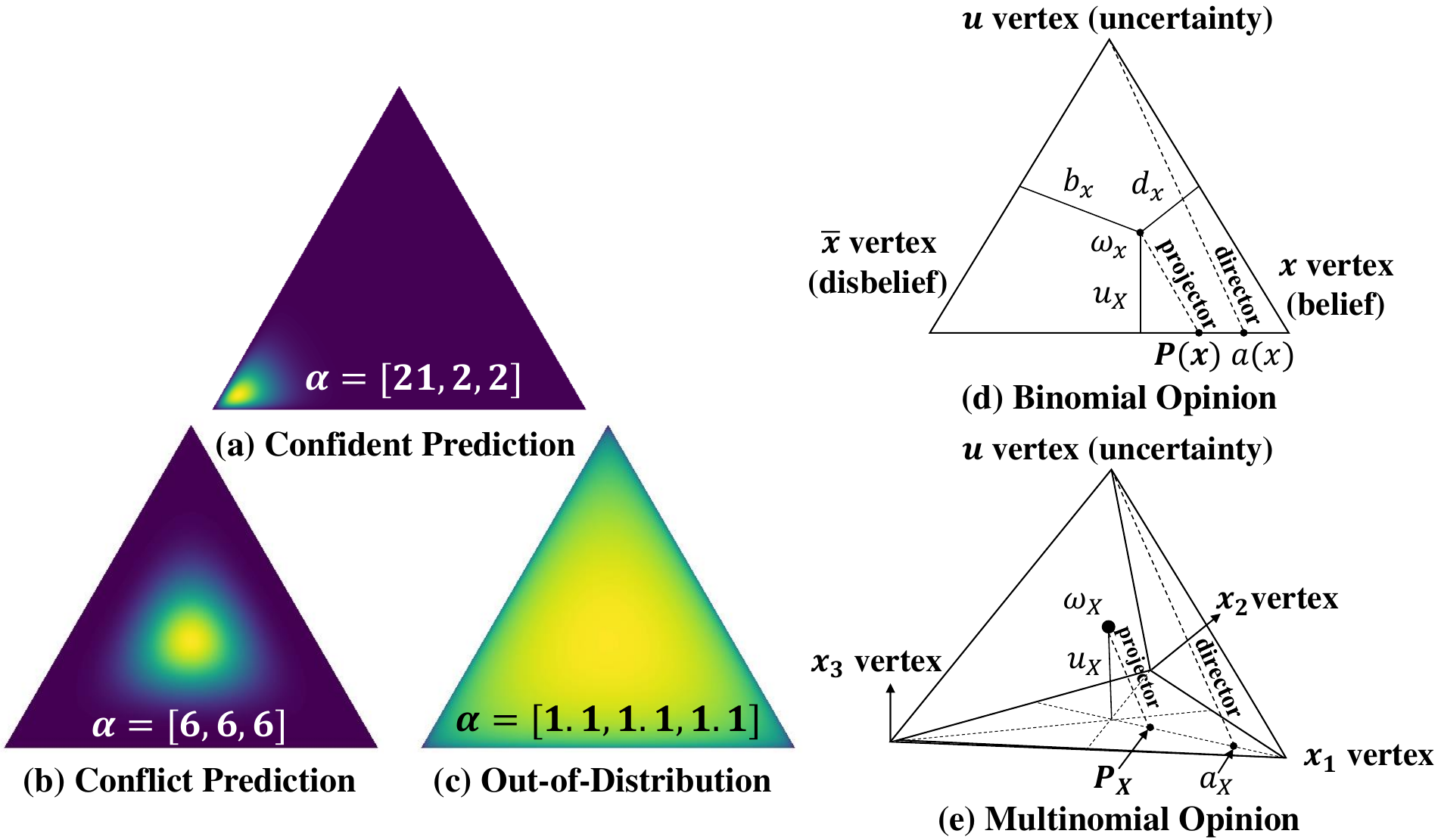}}
\caption{The typical example of Dirichlet distribution and multinomial opinion. (a), (b), (c): Dirichlet distribution with different uncertainty. (d): barycentric triangle visualization of binomial opinion. (e) barycentric tetrahedron visualization of trinomial opinion.}
\label{fig:Opinion_example}
\end{figure}

\subsection{Notation and Problem Definition}
\noindent An attributed and undirected graph can be represented by $G=(V, E)$, where $V=\left\{v_1, v_2, \ldots, v_n\right\}$ is the set of nodes, $E$ is the set of edges in the graph, and $n$ represents the total node amount. $\mathbf{A} \in \mathbb{R}^{n \times n}$ represents the adjacency matrix of $G$, where the $i$, $j$-th entry $a_{i j}=1$ if and only if there exists edge between $v_i$ and $v_j$ in $E$, otherwise $a_{i j}=0$. $\mathbf{X} \in \mathbb{R}^{n \times d}$ and $\mathbf{Y} \in \mathbb{R}^{n \times K}$ represent the node feature matrix and one-hot node label matrix of $G$ respectively, where $d$ is the feature dimension and $K$ is the number of classes. 

In this paper, we regard semi-supervised node classification, a conventional and widely used graph learning problem, as the prime focus of our research. In semi-supervised node classification, we have a set $\mathbf{Y}=\left\{\mathbf{Y}^{UL}, \mathbf{Y}^L\right\}$ where $\mathbf{Y}^{UL}$ and $\mathbf{Y}^L$ represent the labels of unlabeled and labeled nodes, respectively. The objective is to establish a mapping between node features $\mathbf{X}$ and node labels $\mathbf{Y}$. Specifically, the target is to train a prediction function $f\left(G, \mathbf{X}, \mathbf{Y}^L\right) \rightarrow \mathbf{Y}^{UL}$. 

To facilitate the interpretability and trustworthiness of graph learning models, uncertainty-aware GNNs aim to give the principal (i.e. model owner, user, etc.) an assessment of the trustworthiness of their own predictions, so that the principal can understand their predictions. 
Formally, we can reach the goal by the prediction function $f(G,\mathbf{X}, \mathbf{Y}^L) \rightarrow (\mathbf{Y}^{UL}, \mathbf{U})$, where $\mathbf{Y}^{UL}$ is the probability distribution of the prediction and $\mathbf{U}$ is the model's assessment of the uncertainty of the prediction. In practice, the model can be optimized by $\underset{\boldsymbol{\theta}}{\operatorname{argmin}} \sum_{i=1}^n \mathrm{err(\boldsymbol{\theta})+\mathbf{U(\boldsymbol{\theta})}}$, where $err(\cdot)$ and $\boldsymbol{\theta}$ indicates prediction error and learnable weights, respectively.

\subsection{Graph Neural Networks}
\noindent Current GNNs are generally built on the neural message passing framework, performing an embedding propagation (EP) and an embedding transformation (ET) operation at each layer. Taking GCN~\cite{kipf2016semi} as an example, the calculation formula of the ${\ell}$-th layer of GCN can be written by:

\begin{equation}
\mathbf{X}^{\ell}=\delta\left(\hat{\mathbf{A}} \mathbf{X}^{\ell-1} \Theta^{\ell}\right), \quad \hat{\mathbf{A}}=\widetilde{\mathbf{D}}^{-\frac{1}{2}} \tilde{\mathbf{A}} \widetilde{\mathbf{D}}^{-\frac{1}{2}},
\label{eq.1}
\end{equation}

\noindent where $\tilde{\mathbf{A}}=\mathbf{A}+\mathbf{I}$ represents the adjacency matrix of the undirected graph $G$ after adding self-loops matrix $\mathbf{I}$, $\widetilde{\mathbf{D}}$ is the degree matrix, $\mathbf{X}^{\ell}$ is the embedding matrix of nodes at the ${\ell}$-th layer, $\mathbf{X}^{0}=\mathbf{X}$ represents the original feature matrix of nodes. $\Theta^{\ell}$ is the trainable weight for ET, and $\delta$ refers to the activation function, $\hat{\mathbf{A}}$ denotes the adjacency matrix after Laplacian symmetric normalization. 
As shown in Eq.~\ref{eq.1}, EP aggregates the embeddings of node neighborhoods into itself through $\hat{\mathbf{A}} \mathbf{X}^{\ell-1}$, and ET implements nonlinear mapping of nodes embeddings based on $\delta (\mathbf{X}^{\ell-1}\Theta^{\ell})$.
Based on the different combination strategies of EP and ET, GNNs can be divided into coupled and decoupled GNNs. Among them, GCN~\cite{kipf2016semi} and GraphSAGE~\cite{hamilton2017inductive} belong to coupled GNNs, and they nest EP and ET each other, i.e., $EP \to ET \cdots EP \to ET$. For decoupled GNNs, on the one hand, SGC~\cite{wu2019simplifying} executes EP multiple times before executing ET, i.e., $EP \cdots EP \to ET \cdots ET$. APPNP~\cite{gasteiger2018predict}, on the other hand, executes ET and then EP multiple times, that is, $ET \cdots ET \to EP \cdots EP$.

\subsection{Subjective Logic}\label{sec:SL}

\noindent In fact, the uncertainty in GNNs can arise from various factors, including the inherent randomness in the data, errors during model training, and unknown test data from different distributions. These uncertainties are ultimately reflected in the model's predictions, such as probability distributions. The primary goal of this paper is to achieve reliable predictions, improve node classification accuracy, and clarify the risks associated with incorrect predictions. Therefore, we focus on quantifying uncertainty at a higher level (i.e., the model's predictions), rather than delving into specific sources of uncertainty. 
Subjective logic (SL) is an evidence-based formalism for uncertainty reasoning which can be integrated into end-to-end models with Dirichlet distribution~\cite{josang2016subjective}. Contrary to the Bayesian approach, SL allows for a simpler and more direct modeling of uncertainty, focusing on the vacuum of evidence in the predictions. For more details due to the length of the paper, please refer to the reference~\cite{josang2016subjective,han2022trusted}.

In the context of $K$-class classification, SL relates the parameter $\boldsymbol{\alpha}$ of the Dirichlet distribution $\operatorname{Dir}(\mathbf{p}|\boldsymbol{\alpha})$ to the belief mass of the opinion, where the Dirichlet distribution can be considered as the conjugate prior of the categorical distribution. $\mathbf{p}=\left[p_1, \ldots, p_K\right]^T$ represents the simplex of class assignment probabilities~\cite{bishop2006pattern}. Several typical Dirichlet distributions and multinomial opinions are shown in Figure~\ref{fig:Opinion_example}.
Specifically, $\boldsymbol{\alpha}$ can be computed by $\alpha_k=e_k+a_kW$, where $e_k$ refers to the amount of evidence collected from the neural network that supports the sample being classified into the $k$-th class, and $e_k\geq0$, $W$ represent the non-information weight of uncertain evidence and $\mathbf{a}$ refers to the base rate distribution. Without loss of generality, $a_1=\ldots=a_K=1$ and $W=K$, i.e., $\alpha_k=e_k+1$. And the expected probability distribution $\mathbb{E}\left[p_k\right]$ can be calculated by:

\begin{equation}
\mathbb{E}\left[p_k\right]=\frac{\alpha_k}{S}=\frac{e_k+1}{K+\sum_{k=1}^K e_k},
\label{eq.2}
\end{equation}

\noindent where $S=\sum_{k=1}^K\left(e_k+1\right)=\sum_{k=1}^K \alpha_k$ refers to the Dirichlet strength.

Additionally, in a $K$-class problem, the multinomial opinion of the sample is represented as $\boldsymbol{\omega}=(\mathbf{b}, u, \mathbf{a})$, $\mathbf{b}$ and $u$ satisfy the following equation:

\begin{equation}
u+\sum_{k=1}^K b_k=1, u\geq0, b_k\geq0,
\label{eq.3}
\end{equation}

\noindent where $u$ represents the uncertainty mass of the opinion, $\mathbf{b}=[b_1,\ldots,b_K]^T$ and $\mathbf{a}=[a_1,\ldots,a_K]^T$ represent belief mass distribution and base rate distribution, respectively. Then, the projected probability distribution of a multinomial opinion can be calculated by $p_k=b_k+a_ku$. On this basis, when the expected probability distribution of the Dirichlet distribution is equal to the opinion projected probability distribution, the mapping between the opinion and the Dirichlet distribution can be obtained by $\boldsymbol{\omega}=(\mathbf{b}, u, \mathbf{a}) \leftrightarrow \operatorname{Dir}(\mathbf{p}|\boldsymbol{\alpha})$, satisfying:

\begin{equation}
b_k=\frac{e_k}{S}=\frac{\alpha_{k}-1}{S}, u=\frac{K}{S}.
\label{eq.4}
\end{equation}

To sum up, with SL, we can realize the transformation of Dirichlet distribution to multinomial opinion and quantify the uncertainty of evidence.

%% file: 3_Moti.tex
In this section, we expose the influence of \textbf{model depth} on the prediction uncertainty of uncertainty-aware GNNs. Firstly, we explore the optimal depth to acquire the most confident prediction for each class of nodes. Then, we investigate the impact of increasing network depth on prediction uncertainty. Drawing from analytical and empirical studies, we try to summarize key design principles for uncertainty-aware GNNs with the consideration of model depths.

\subsection{Optimal Class-Level Confidence at Different Depths}
\label{sec:class_moti}
In the existing uncertainty-aware GNNs, a common paradigm is to estimate the prediction uncertainty from the final output layer. In this case, the prediction uncertainties of all nodes are limited by the fixed number of layers. More specifically, when employing a constant number of embedding propagation (EP) steps, the receptive fields that govern the prediction uncertainties for individual nodes remain unchanged. Then, a natural question arises: ``\textit{Does the optimal depth to achieve the optimal confidence for each node remain constant?}''

The answer to this question, intuitively, is negative. In real-world scenarios, a reasonable hypothesis is that different nodes may depend on different numbers of neighboring hops to make optimal predictions with high confidence~\cite{wang2022uncovering}. Typically, nodes with different categorical labels would have different optimal depths to make confident predictions. Taking co-purchase networks as an example (as shown in Figure~\ref{fig:Structural_unfairness}), for several classes, it only requires 1-hop information propagation to obtain the confident prediction, such as the accessories - Class A in Figure~\ref{fig:Structural_unfairness}~(b); differently, for other classes (e.g., cell phones - Class B), the optimal GNN depths to make a reliable prediction can be larger, as demonstrated in Figure~\ref{fig:Structural_unfairness}~(c). 

To empirically verify the above hypothesis, we conduct a motivated experiment on real-world graph data. Specifically, we use the true class probability as the confidence measure of each sample~\cite{corbiere2019addressing, han2022multimodal}, and visualize the confidence distribution of 6 different classes of nodes at different depths on Citeseer dataset. In Figure~\ref{fig:Confi_Dis_Citeseer}, we can witness that the optimal depths to acquire the highest confidence for each class are quite different. 
For instance, in classes B, C, and F, an 8-layer GNN has the highest confidence. However, the highest confidence for classes D and E occurs when the depth is 4 and 10, respectively. For class A, differently, the best confidence is achieved by a shallow GNN with 2 layers.

\noindent\textbf{Discussion:} Through empirical analysis, we discover that the variation trend of the confidence distribution for each class does not follow a completely uniform pattern across different depths or neighborhoods. In addition, the depth of the best confidence will be different for each class. Hence, two essential criteria can be identified: 1) To achieve confident predictions for various nodes, it is advantageous to incorporate the GNN's output predictions at multiple depths. 2) To comprehensively obtain confident predictions for nodes across different classes, a suitably deep GNN (i.e., with an adequate number of EP steps) is required. 

\begin{figure}[tbp]
\centerline{\includegraphics[width=0.5\textwidth]{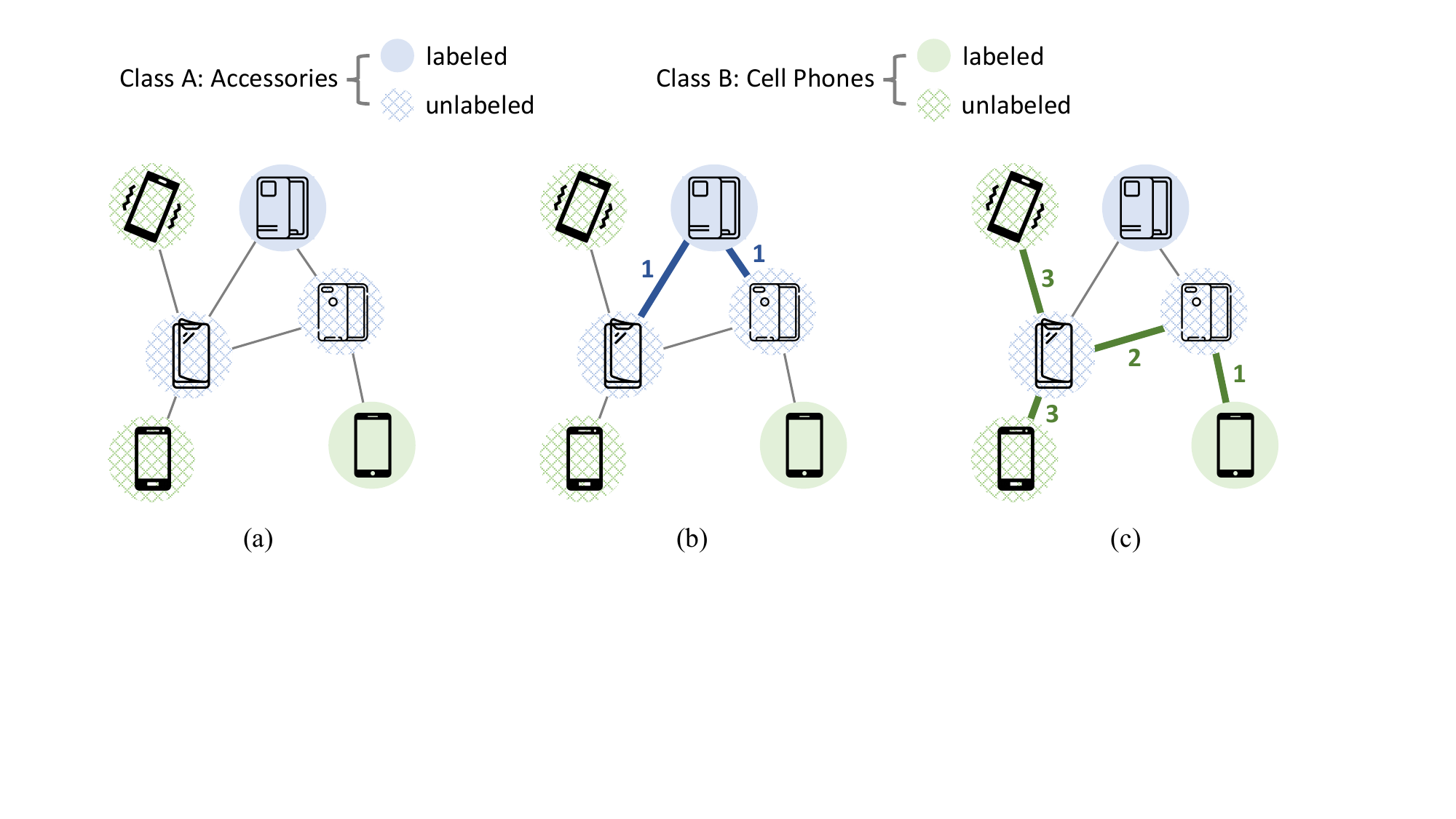}}
\caption{An example of different optimal class-level certainty at GNNs with different depths.}
\label{fig:Structural_unfairness}
\end{figure}

\begin{figure*} [tbp]
	\centering
	\subfloat[Class A\label{subfig:con_dis_citeseer_c0}]{
		\includegraphics[scale=0.15]{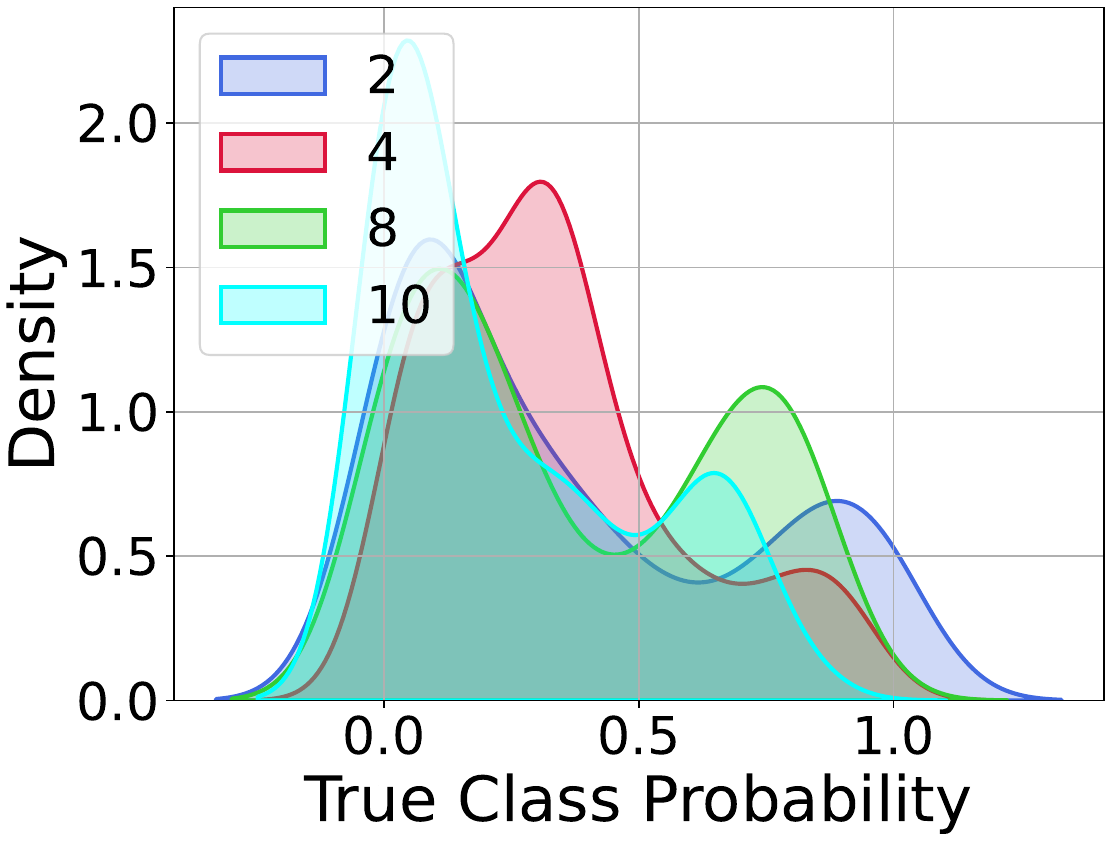}}\hfill
	\subfloat[Class B\label{subfig:con_dis_citeseer_c1}]{
		\includegraphics[scale=0.15]{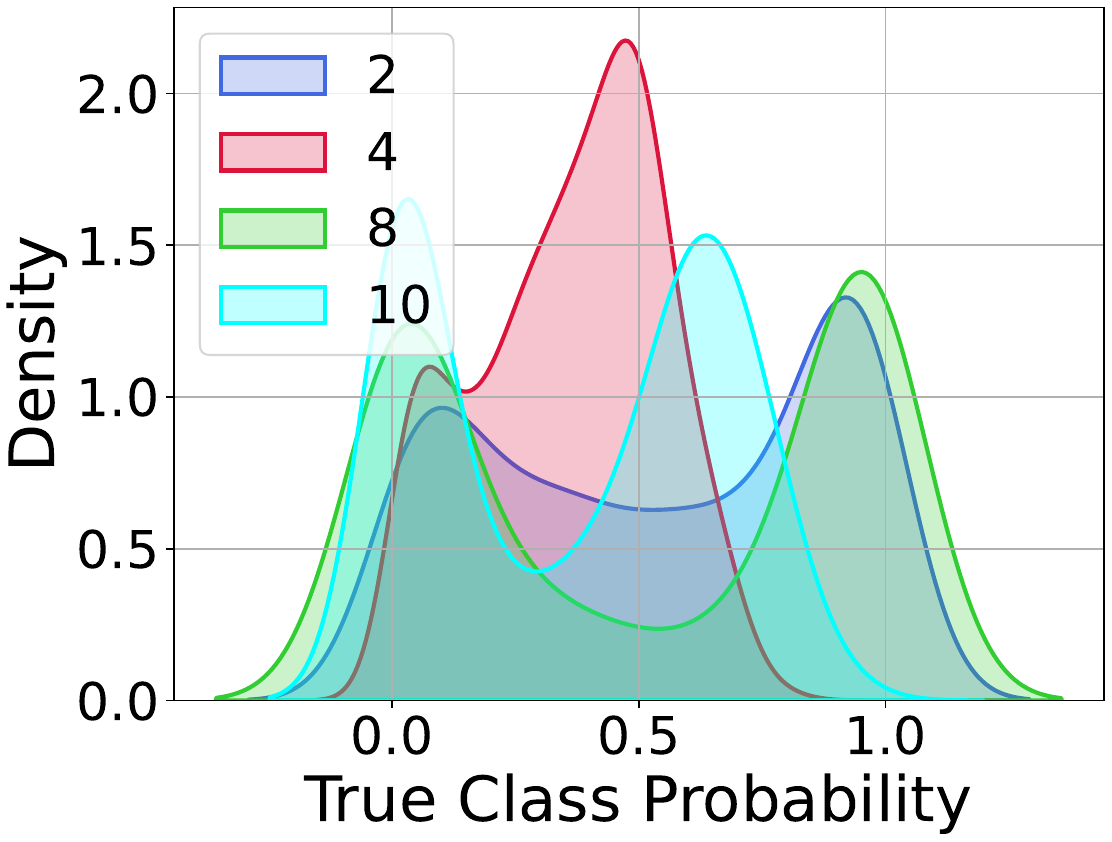}}\hfill
	\subfloat[Class C\label{subfig:con_dis_citeseer_c2}]{
		\includegraphics[scale=0.15]{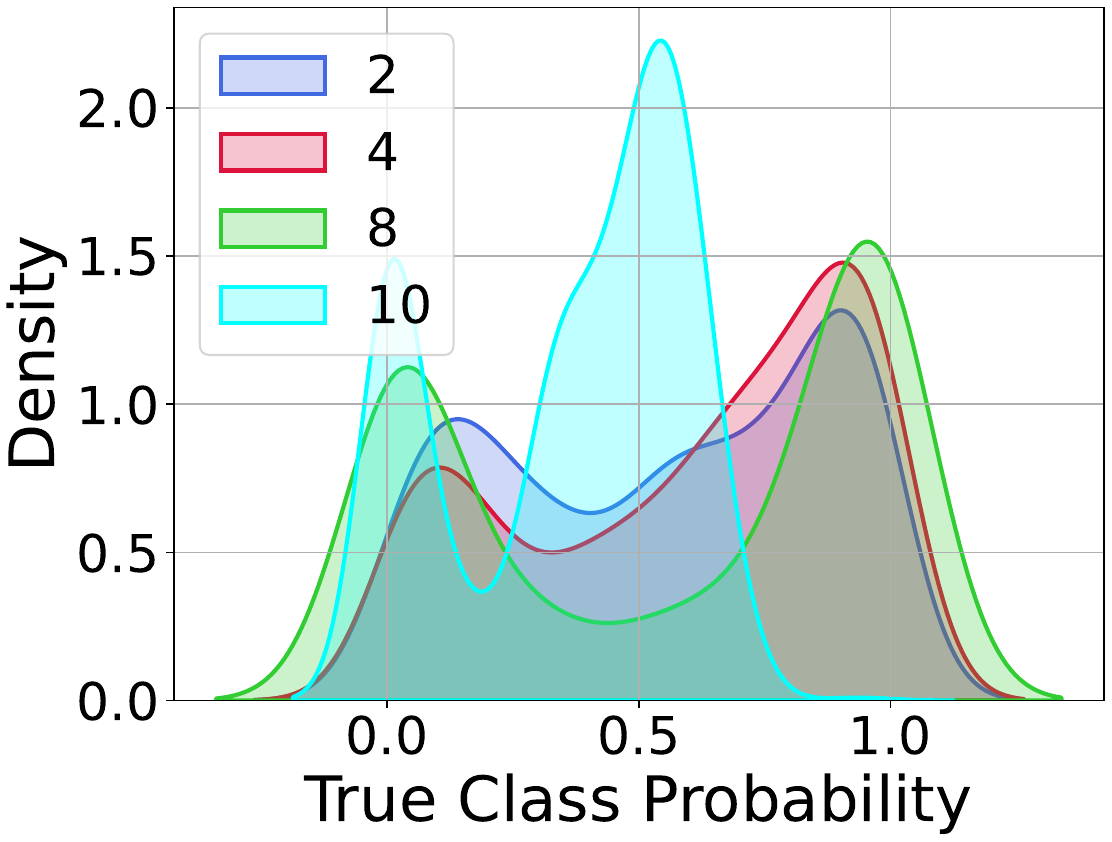}}\hfill
        \subfloat[Class D\label{subfig:con_dis_citeseer_c3}]{
		\includegraphics[scale=0.15]{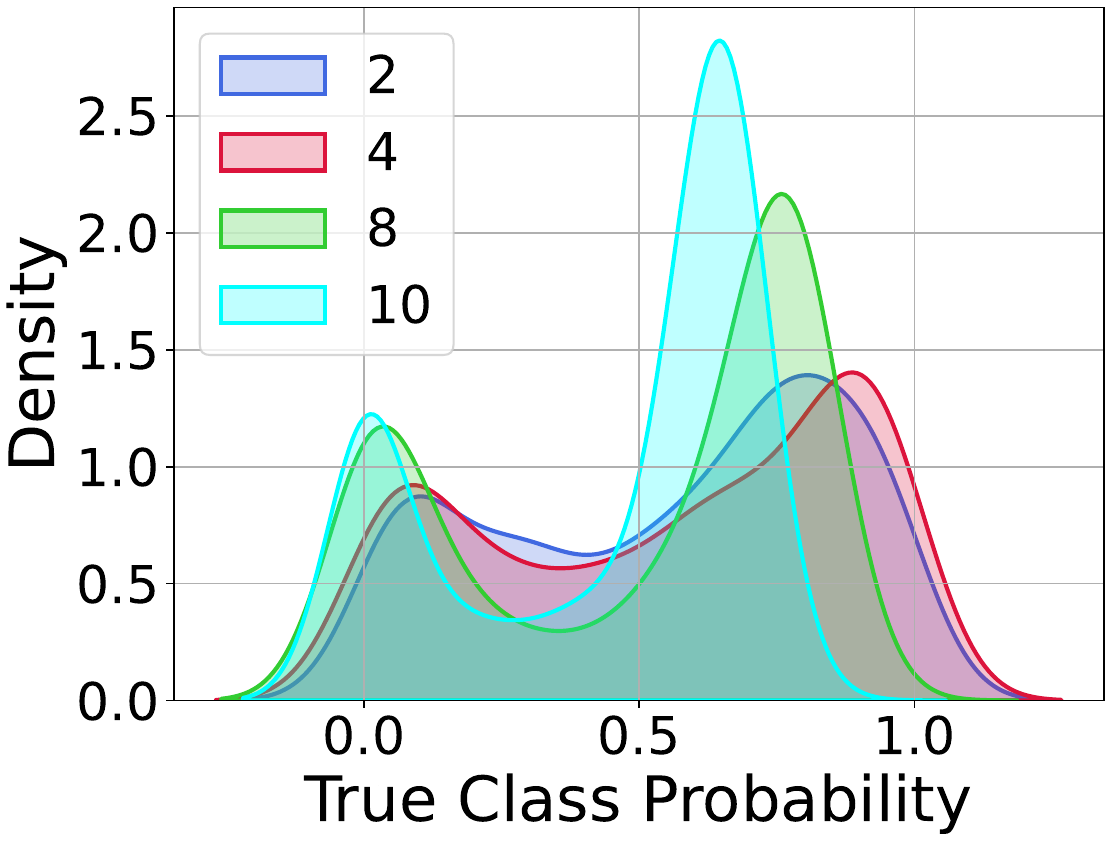}}\hfill
	\subfloat[Class E\label{subfig:con_dis_citeseer_c4}]{
		\includegraphics[scale=0.15]{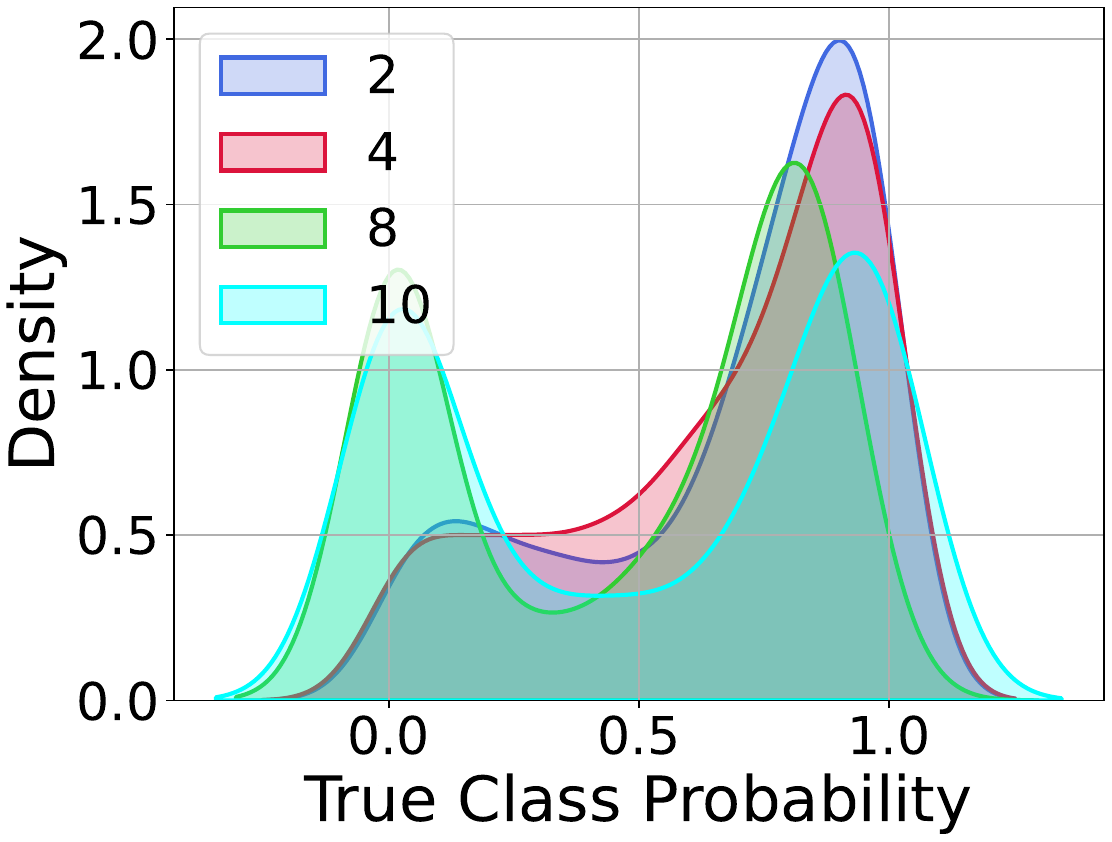}}\hfill
	\subfloat[Class F\label{subfig:con_dis_citeseer_c5}]{
		\includegraphics[scale=0.15]{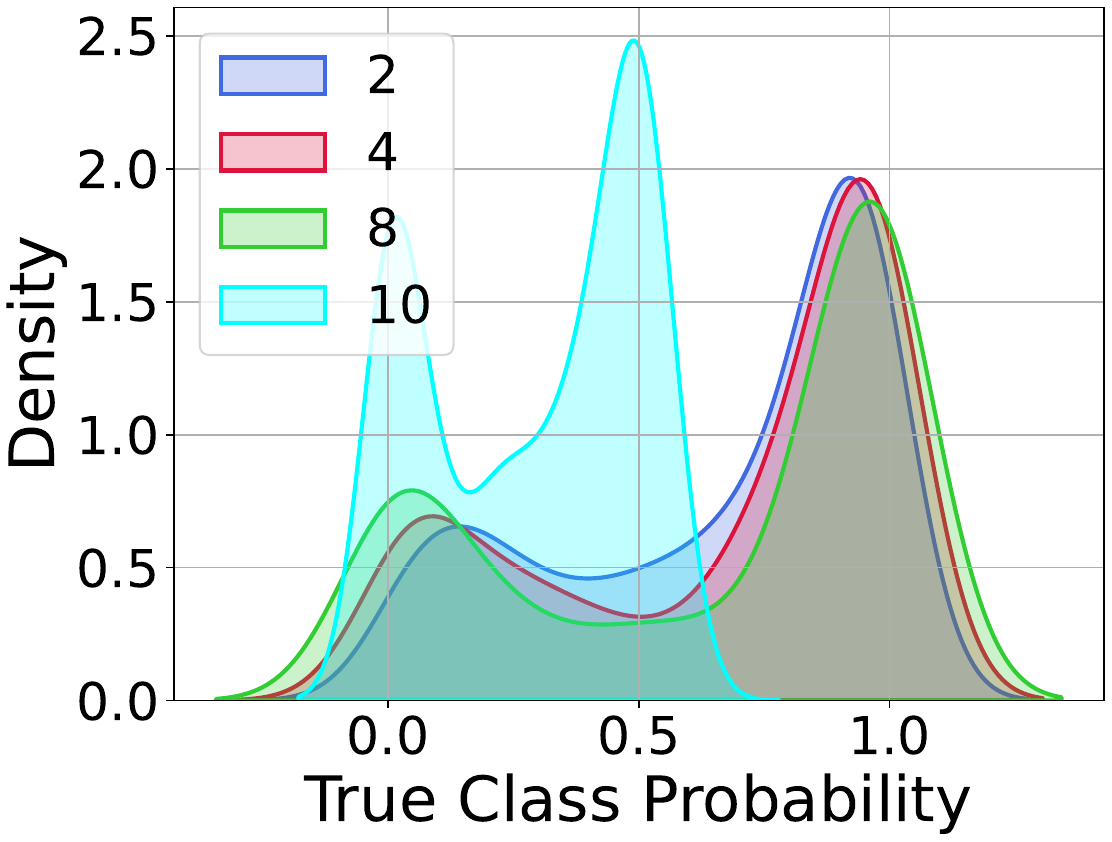}}
	\caption{The distributions of confidence w.r.t. true class probability for 6 classes (A-F) of Citeseer dataset at different depths.}
	\label{fig:Confi_Dis_Citeseer}
\end{figure*}

\begin{figure*} [tbp]
	\centering
	\subfloat[Cora\label{subfig:std_cora}]{
		\includegraphics[scale=0.18]{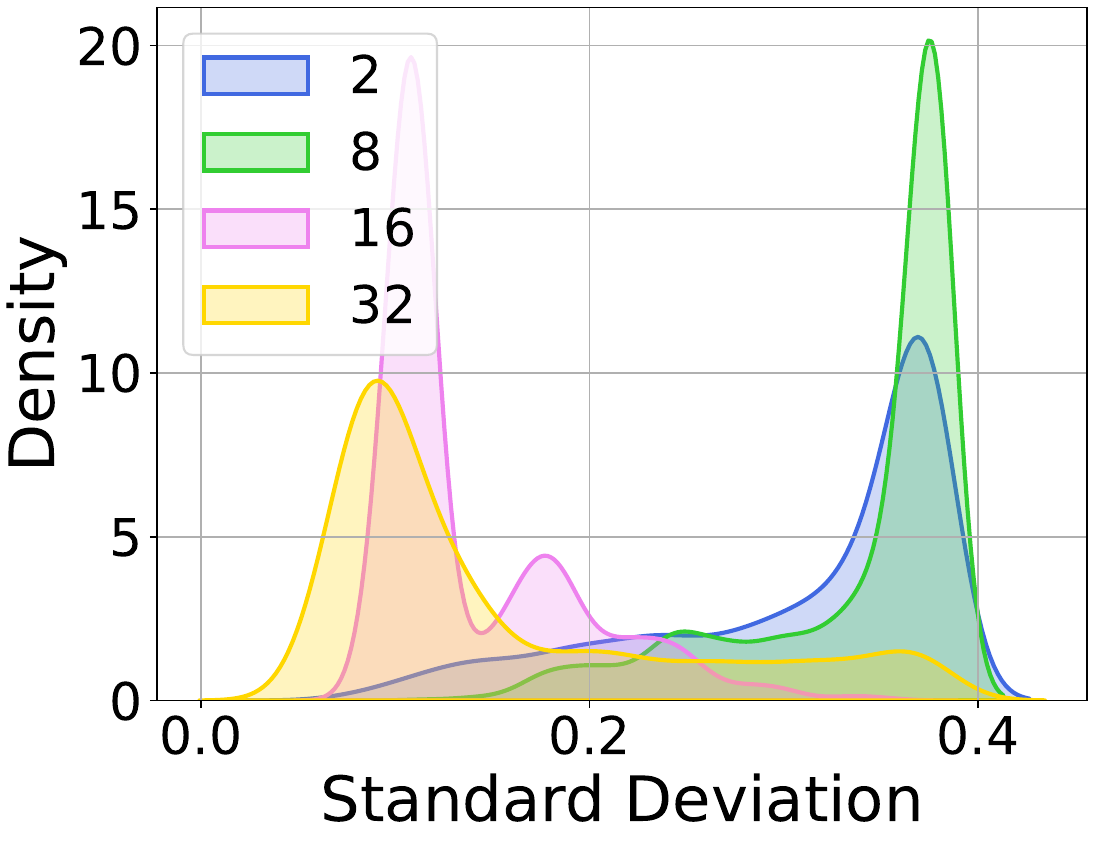}}\hfill
	\subfloat[Citeseer\label{subfig:std_citeseer}]{
		\includegraphics[scale=0.18]{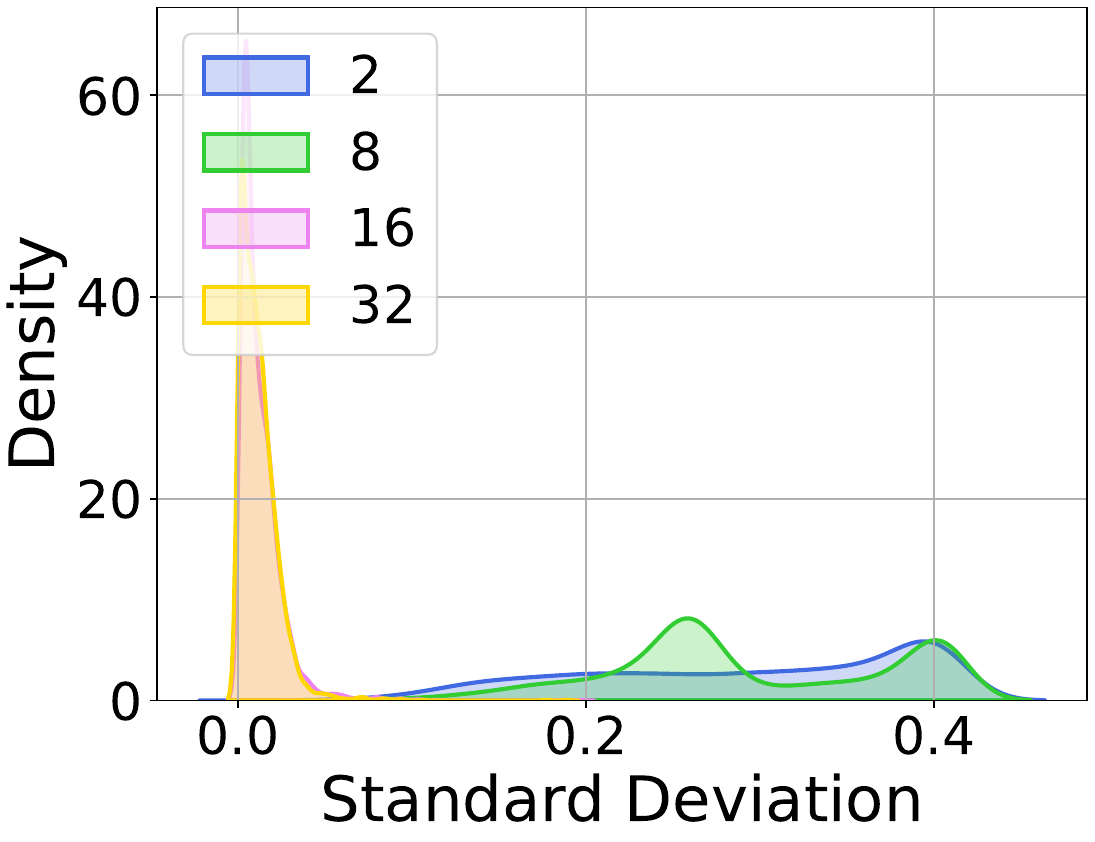}}\hfill
	\subfloat[Pubmed\label{subfig:std_pubmed}]{
		\includegraphics[scale=0.18]{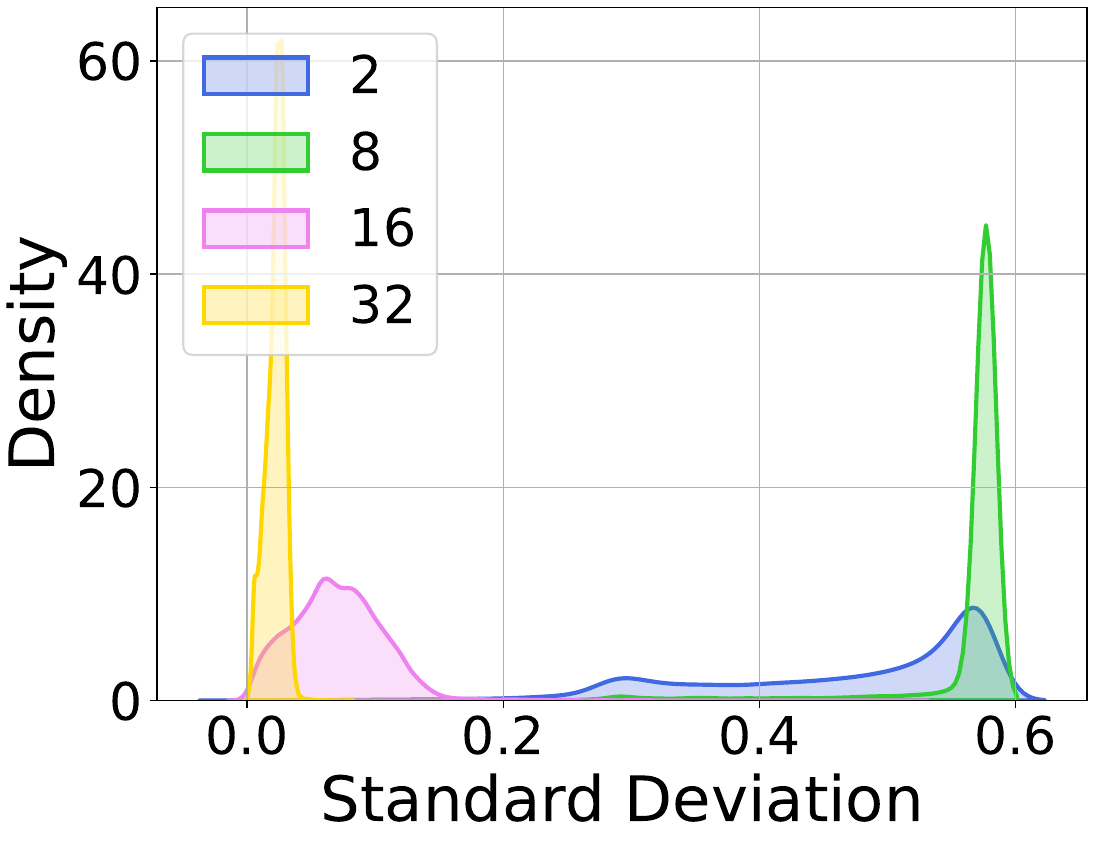}}\hfill
	\subfloat[Computers\label{subfig:std_computers}]{
		\includegraphics[scale=0.18]{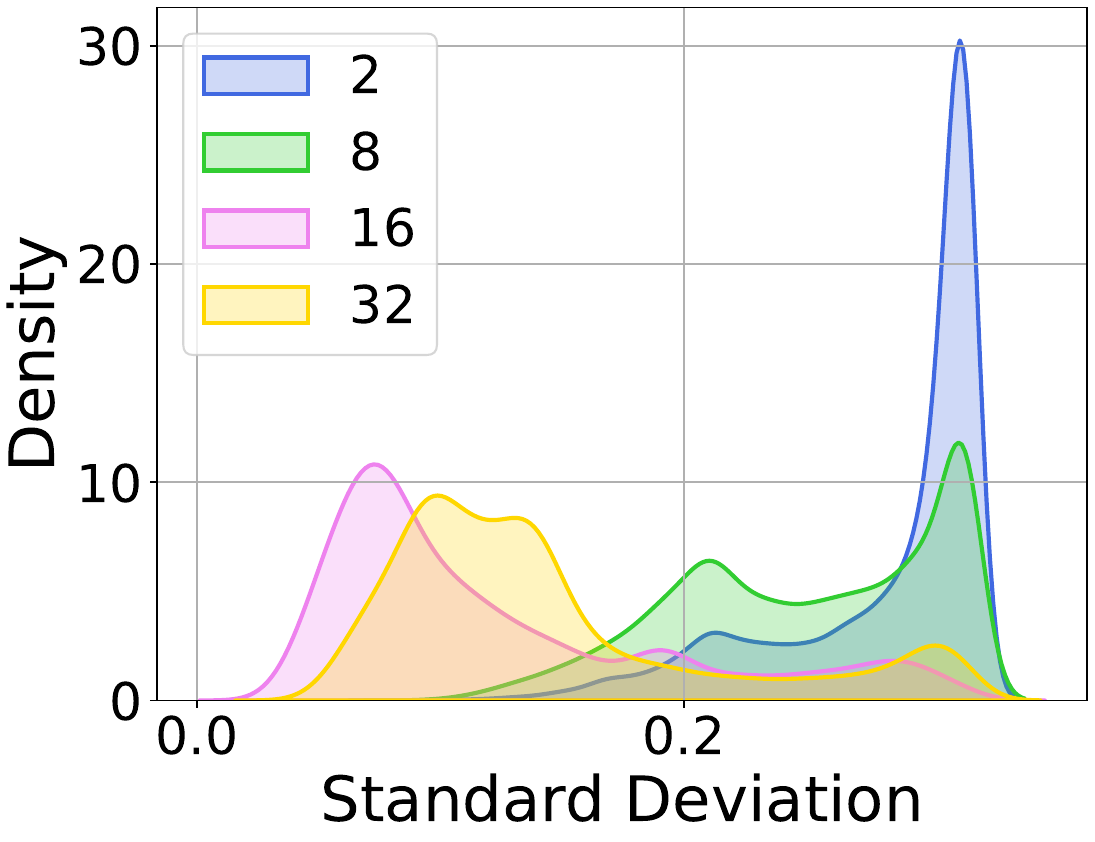}}\hfill
	\subfloat[Photo\label{subfig:std_photo}]{
		\includegraphics[scale=0.18]{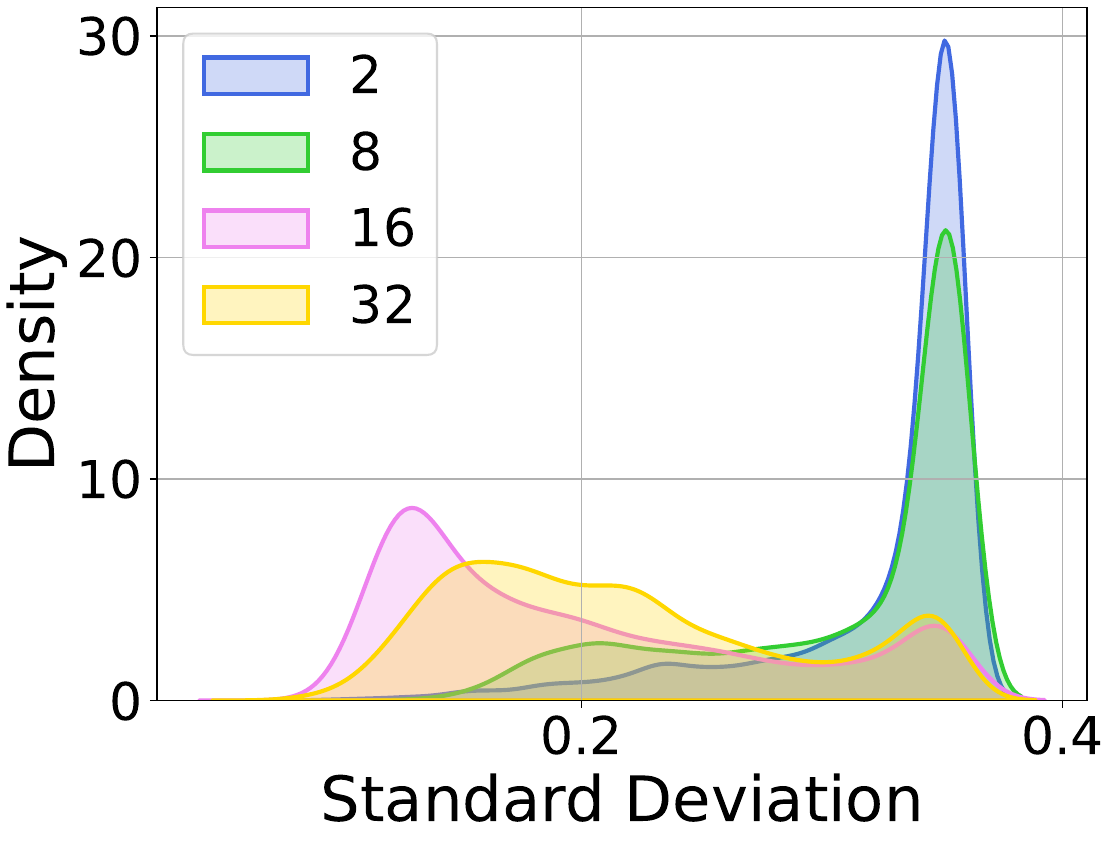}}
	\caption{The distributions of confidence w.r.t. standard deviation density of class probability at different depths on 5 datasets.}
	\label{fig:Std_Distribution}
\end{figure*}

\subsection{Low Confidence of Over-Deep GNNs}

From the above discussion, we require a deeper GNN with sufficient EP steps to cover the optimal depths for all nodes. However, \textit{is that always beneficial to deepen a uncertainty-aware GNN?} From the perspective of representation quality, GNNs with excessive EP layers usually suffer from the over-smoothing issue, hindering their prediction accuracy~\cite{chen2020measuring}. Therefore, we also have reason to conjecture 
that when the number of EP is too large, it may cause conflicts or out-of-distribution (OOD) predictions due to over-smoothing~\cite{rusch2023survey}.

To verify our hypothesis, we conducted an experiment investigating the impact of increasing the depth of the GNNs. Specifically, we set the depth of a GCN to 2, 8, 16, and 32 respectively in this empirical analysis~\cite{zhou2021dirichlet}. The standard deviation of the class probability distribution serves as the metric of uncertainty, and the visualization is shown in Figure~\ref{fig:Std_Distribution}. As we can see in the figure, for GCNs with a depth of 2, the standard deviation of the final class probability distribution of nodes is concentrated in a relatively higher area. This means that the probability values of different classes of nodes are well-separated, enabling the model to make confident and correct predictions. However, as the number of layers increases, the output node class probability distribution becomes less clear. Specifically, the standard deviation of the class probability distribution for 16-layer and 32-layer GCNs gradually shifts to a relatively smaller area. At this stage, the probabilities of each class of nodes are almost equal, and the model struggles to make confident predictions, potentially leading to unreliable predictions. In short, as the model gets deeper, the prediction uncertainty of GCN model shows an increasing trend, making it difficult for the model to provide highly reliable prediction, and eventually leading to significant performance degradation. In contrast, shallow GNNs, primarily focusing on lower-order neighborhoods, tend to produce more confident predictions. 

\noindent \textbf{Discussion:} According to the empirical analysis in Section~\ref{sec:class_moti}, we summarize that the optimal depth for generating the most confident predictions, as determined by the number of EP, varies for different nodes. 
To cover the optimal depths for all nodes, a simple approach is to choose a large enough range of propagation. However, in further experiments we found that excessively large depth of network may prodced more predictions with high uncertainty, resulting in conflicts or OOD issues of prediction. On these insights, it becomes imperative to devise suitable strategies that ensure prediction confidence. Therefore, our methodology design will predominantly revolve around the modeling and exploitation of prediction uncertainty from various neighborhoods to enable the most accuracy and trustworthy prediction.

%% file: 4_Method.tex
In order to enable GNN to make reliable and confident predictions by adaptively exploiting varying hops of neighborhoods, in this section, we introduce a novel uncertainty-aware GNN model termed \underline{\textbf{E}}vidence \underline{\textbf{F}}using \underline{\textbf{G}}raph \underline{\textbf{N}}eural \underline{\textbf{N}}etwork (\ourmethod for short). 
Our core idea is to learn a models that makes predictions at multiple level, each with an associated uncertainty evaluation, which are then fused to generate the final prediction.
The pipeline of \ourmethod is illustrated in Figure~\ref{fig:pipeline}. At the highest level, \ourmethod is composed of two components: \textbf{multi-hop evidence graph learning} that generates accuracy predictions with uncertainties for each node from its neighbors at multiple hops (Sec.~\ref{VA}), and \textbf{cumulative belief fusion (CBF)-based evidence fusion} that combines the multi-hop predictions into a final confident prediction (Sec.~\ref{VB}). To effectively optimize \ourmethod, we introduce three well-crafted loss functions in Sec.~\ref{VC}. Finally, we provide theoretical analyses of the validity of CBF (Sec.~\ref{VD}) and analyze the time complexity of \ourmethod (Sec.~\ref{VE}). 

\begin{figure*}[tbp]
\centerline{\includegraphics[scale=0.55]{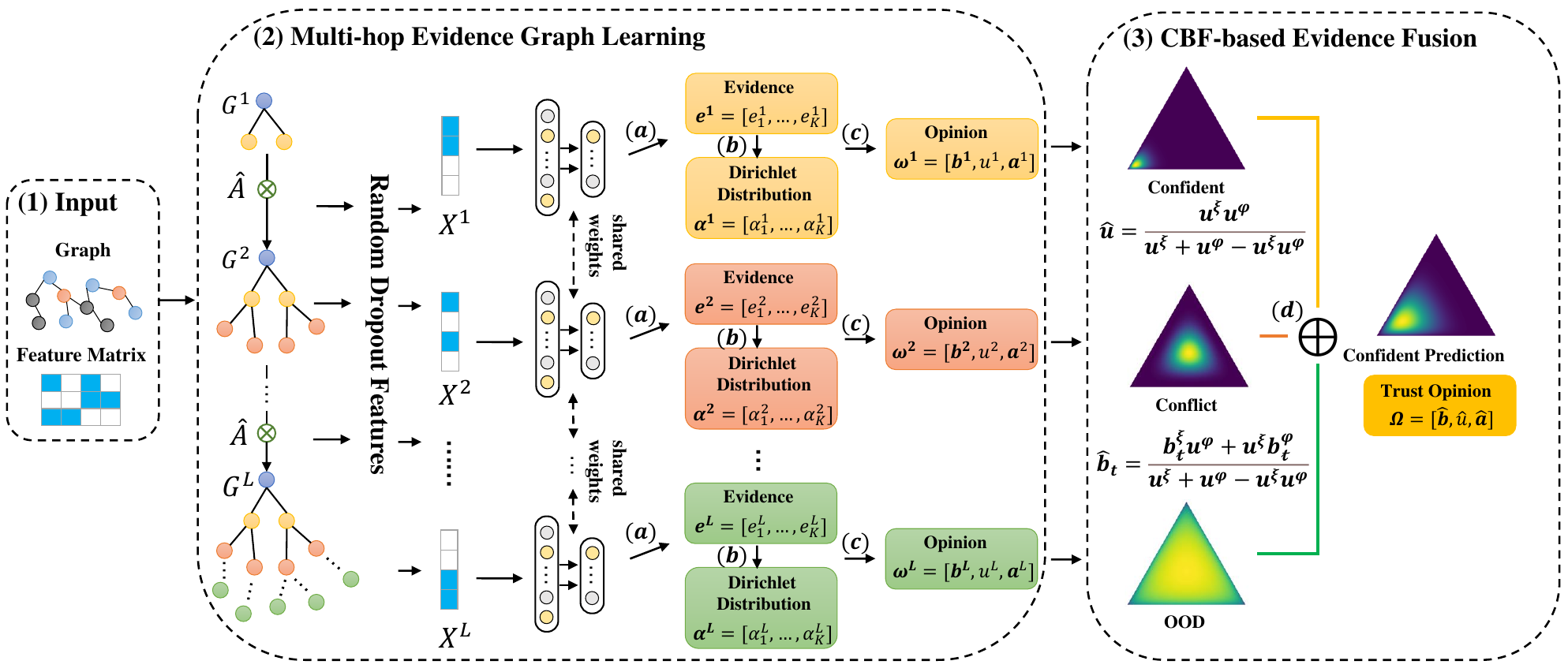}}
\caption{Illustration of proposed method. 
Predictions with associated uncertainties are generated at multiple depths and then fused.
The overall model flow consists of the following steps. (1) Input. Data input, including graph and node features. (2) Multi-hop Evidence Graph Learning. Aggregate node features from different neighborhoods and perform random drops. (a) Then, the aggregate features obtained by different neighborhoods are passed through the corresponding shared evidence generation network to generate evidence vectors supporting each class. (b): Next, we use the evidence to parameterize the parameter ${\alpha}$ of the Dirichlet distribution $\operatorname{Dir}(\mathbf{p} \mid \boldsymbol{\alpha})$ to modeling class probability distribution. (c): The parameterized Dirichlet distribution is mapped to a multinomial opinion to obtain the uncertainty. (3) CBF-based Evidence Fusion. (d): Dynamically fusing predictions with uncertainty in each neighborhood of node based on CBF to maximize the trustworthiness of the final prediction.}
\label{fig:pipeline}
\end{figure*}

\subsection{Multi-hop Evidence Graph Learning}\label{VA}

In Section~\ref{DESIGN}, we conclusively recognize the importance of incorporating multi-hop neighborhood information to simultaneously ensure confident prediction for nodes across diverse classes. To this end, we design a GNN with multi-step propagation to generate node features at different depths; after that, an evidence learning module is designed to obtain evidence from different neighborhoods of nodes. 

To learn graph representations from different hops of neighborhoods, a simple solution is to stack multiple GNN layers. 
However, coupled GNN architecture may suffer from model degradation due to too many ET steps and high model complexity~\cite{zhang2022model}. Therefore, in \ourmethod, we employ a decoupled GNN architecture where EP/ET operations are separated. In specific, multiple steps of feature propagation are performed on the raw input features before conducting feature transformation. The $\ell$-{th} EP can be calculated as follows:
\begin{equation}
\mathbf{X}^{\ell}=\hat{\mathbf{A}} \mathbf{X}^{\ell-1}.
\label{eq.5}
\end{equation}

As discussed before, nodes have different optimal depths, which indicates the strong dependence of nodes on specific neighborhood information. Nevertheless, the model can be over-sensitive to specific features during model training, which hinders the overall performance. To mitigate this sensitivity, we introduce a simple data augmentation strategy, random perturbation, which is applied after the feature propagation step. Specifically, for $\ell$-{th} layer node embedding $\mathbf{X}^\ell$, the perturbed embedding $\tilde{\mathbf{X}}^{\ell}$ is calculated as follows:
\begin{equation}
    \tilde{\mathbf{X}}^{\ell}=\frac{\mathcal{b}}{1-\sigma} \mathbf{X}^{\ell},
\label{eq.62}
\end{equation}

\noindent where $\mathcal{b} \sim \mathrm{Bernoullli}(1-\sigma)$ is the binary mask set of nodes, and $(1-\sigma)$ is the scaling factor to ensure that the expectation of the embedding before and after the perturbation are equal. The perturbed strategy ensures each node has only subsets of information under different neighborhoods, which reduces the dependence on specific neighborhoods and makes the evidence generation more robust.

After multi-step propagation, now we aim to generate the predictions with uncertainty on the basis of propagated features. In \ourmethod, we use a shared evidence generation network to obtain the evidence of nodes in different neighborhoods. In fact, the evidence generation network can be regarded as the ET in our model. According to previous studies~\cite{van2020uncertainty,moon2020confidence}, employing the conventional Softmax function output as the confidence level of prediction often results in overconfidence, especially in cases of errors, as the largest Softmax output determines the final prediction. Additionally, as discussed in Section~\ref{DESIGN}, excessively large EP steps can generate low-quality embeddings. In such situations, evidence theory can differentiate between conflicting predictions and OOD cases (e.g., Figure~\ref{fig:Opinion_example}). This is in contrast to the Softmax function, which mandates a probability distribution summing to $1$ for each class of the sample.
Thus, in order to generate evidence $e_k\geq0$, supporting each class, we need to modify the neural network used for classification~\cite{sensoy2018evidential}. In particular, we replace the Softmax function used for classification with any activation function (e.g., SoftPlus function). Formally, for the node $v_i$ feature vector $\tilde{\mathbf{X}}_i^{\ell}$ of the ${\ell}$-{th} propagation, the node $v_i$ evidence $\mathbf{e}^\ell_i$ is calculated as follows:

\begin{equation}
\mathbf{e}^\ell_i=f_{\Theta}^e\left(\tilde{\mathbf{X}}^{\ell}_i\right)=\left[e_{i 1}^{\ell}, e_{i 2}^{\ell}, \ldots, e_{i K}^{\ell}\right],
\label{eq.6}
\end{equation}

\noindent where $f_{\Theta}^e$ is the shared evidence generation network with a SoftPlus non-linear operation and $e_{i K}^\ell$ refers to the amount of evidence supporting class $K$-{th} obtained by node $v_i$ under the $\ell$-{th} propagation. 

On this basis, according to the theoretical framework in Section~\ref{sec:SL}, we can obtain the $\ell$-{th} layer's opinion $\boldsymbol{\omega}^\ell_i=(\mathbf{b}^\ell_i, u^\ell_i, \mathbf{a}^\ell_i)$, where belief mass $b_{i k}^{\ell}=e_{i k}^{\ell} / S^{\ell}_i$ of different classes and the uncertainty mass $u^{\ell}_i=K / S^{\ell}_i$ of the $K$-class problem through the evidence $\mathbf{e}^\ell_i$. Especially, we can obtain probability distribution ${\mathbf{p}^\ell_i}=\left[{p}^\ell_{i 1}, \ldots, {p}^\ell_{i K}\right], {p}^\ell_{i K}={b^\ell_{i K}}+{a^\ell_{i K}} {u}^\ell_i$ to serve as the unified prediction of our model, with ${u}^\ell_i$ as the prediction uncertainty. As more evidence supporting the $k$-{th} class is observed, the belief mass of assigning the $k$-{th} class increases, leading to a higher probability for that class. In contrast, as the total evidence observed decreases, the overall uncertainty increases, resulting in higher uncertainty mass. 
In the similar way, the node evidence space $\mathbf{e}=\left[\mathbf{e}^1, \mathbf{e}^2, \ldots, \mathbf{e}^L\right]$ for a total of $L$ propagation can be obtained. Correspondingly, the opinion space $\boldsymbol{\omega}=\left[\boldsymbol{\omega}^1, \boldsymbol{\omega}^2, \ldots, \boldsymbol{\omega}^L\right]$.

\subsection{CBF-based Evidence Fusion}\label{VB}
Via multi-hop evidence learning, we can acquire the node evidence at different layers and the opinion with prediction uncertainty. Nevertheless, as we discussed in Section~\ref{DESIGN}, not all layer outputs can yield reliable predictions. Hence, to uniformly leverage the predictions at different layers to produce the optimal predictions, in this subsection, we introduce a CBF-based evidence fusion mechanism that integrates multiple predictions into the most confident one. 

The core idea of CBF is to allow the combination of evidence from different neighborhoods, resulting in a confidence that considers all available evidence. In concrete, let $\boldsymbol{\omega}^L_i$ be $L$-{th} layer’s opinion over the node $v_i$ in GNNs, the joint opinion $\boldsymbol{\Omega}_i$ can be expressed by:

\begin{equation}
\boldsymbol{\Omega}_i=\boldsymbol{\omega}^1_i \oplus \boldsymbol{\omega}^2_i \oplus \ldots \oplus \boldsymbol{\omega}^L_i.
\label{eq.7}
\end{equation}

Then, we can implement the CBF operator as the simple addition of evidence parameters~\cite{josang2018categories}. In more details, an easy calculation rule (the original fusion rules can be found in the Eq.~\ref{eq.15}) can be formulated as follows:

\begin{equation}
\hat{e}_{i k}=\sum_{\ell=1}^L e_{i k}^{\ell}, \quad \hat{a}_{i k}=\frac{1}{K},
\label{eq.8}
\end{equation}

\noindent where $\hat{e}_{i k}$ is the joint evidence. Then, the joint belief mass $\hat{b}_{i k}$ and joint uncertainty mass $\hat{u}_i$ can be obtained:

\begin{equation}
\hat{b}_{i k}=\frac{\hat{e}_{i k}}{\hat{S}_i}=\frac{\hat{\alpha}_{i k}-1}{\hat{S}_i}, \quad \hat{u}_i=1-\sum_{k=1}^K \hat{b}_{i k},
\label{eq.9}
\end{equation}

\noindent where $\hat{S}_i=\sum_{k=1}^K \hat{\alpha}_{i k}$ is the joint Dirichlet strength. 

According to the CBF-based evidence fusion, for a given node $v_i$, the more trustworthy joint opinion $\boldsymbol{\Omega_i}=(\hat{\mathbf{b}}_i, \hat{u}_i, \hat{\mathbf{a}}_i)$ and the joint Dirichlet distribution parameter $\hat{\boldsymbol{\alpha}}_i$ can be obtained. The probability can be computed from 
$\mathbb{E}\left[p_{i j}\right]=\hat{p}_{i j}=\frac{\hat{\alpha}_{i j}}{\hat{S}_i}=\hat{b}_{i j}+\hat{a}_{i j} \hat{u}_i$, and $\hat{u}_i$ is the prediction uncertainty of that node $v_i$. 
In Section~\ref{VD}, we further discuss the advantages of CBF from the theoretical perspective of uncertainty and accuracy. 

\subsection{Model Training and Optimization}\label{VC}

In this subsection, we introduce the objective function to optimize \ourmethod. The objective function is composed of three loss terms, including evidence cross-entropy (ECE) loss, dissonance coefficient loss, and Kullback-Leibler (KL) divergence loss. 

In order to optimize \ourmethod to make accurate predictions for node classification tasks, we establish the ECE loss as our essential loss function. Specifically, for feature $\tilde{\mathbf{X}}_i$ of node $v_i$, the model gets the joint opinion $\boldsymbol{\Omega_i}=(\hat{\mathbf{b}}_i,\hat{u}_i,\hat{\mathbf{a}}_i)$, and derives $\operatorname{Dir}\left(\mathbf{p}_i \mid \hat{\boldsymbol{\alpha}}_i\right)$. Then there is a probability $\hat{p}_{i j}$ and the label $\mathbf{Y}_i$. Accordingly, for joint Dirichlet distribution $\boldsymbol{\hat{\alpha}_i}$, the ECE loss can be calculated as follows:

\begin{equation}
\begin{gathered}
\mathcal{L}_{ECE}\left(\hat{\boldsymbol{\alpha}}_i\right)=\int\left[\sum_{j=1}^K-y_{i j} \log \left(\hat{p}_{i j}\right)\right] \frac{1}{\mathcal{B}\left(\hat{\boldsymbol{\alpha}}_i\right)} \prod_{j=1}^K \hat{p}_{i j}{ }^{\hat{\alpha}_{i j}-1} d \hat{\mathbf{p}}_i \\
=\sum_{j=1}^K y_{i j}\left(\psi\left(\hat{S}_i\right)-\psi\left(\hat{\alpha}_{i j}\right)\right),
\end{gathered}
\label{eq.10}
\end{equation}

\noindent where $\psi(\cdot)$ is the digamma function. Essentially, $\mathcal{L}_{ECE}$ guarantees that correct class labels accumulate more evidence. 

Moreover, there may be in-distribution but conflicting prediction in the joint opinion of samples as shown in Figure~\ref{fig:Opinion_example}~(b) (i.e., $\hat{b}_1=\ldots=\hat{b}_K$). In this case, the model is still unable to make a confident prediction. To quantify the degree of conflict between the beliefs belonging to different class labels, we further leverage a measure termed dissonance coefficient~\cite{josang2018uncertainty}. Then, the dissonance coefficient can be introduced as a regularization term to reduce the conflict between different class labels of beliefs in the joint opinion $\boldsymbol{\Omega}_i$. 
In particular, the dissonance loss can be calculated as follows:

\begin{equation}
\mathcal{L}_{Dis}\left(\hat{\boldsymbol{\alpha}}_i\right)=\sum_{j=1}^K \frac{\hat{b}_{i j} \sum_{q \neq j} \hat{b}_{i j} \operatorname{Bal}\left(\hat{b}_{i q}, \hat{b}_{i j}\right)}{\sum_{q \neq j} \hat{b}_{i q}},
\label{eq.11}
\end{equation}

\noindent where, the relative mass balance $\operatorname{Bal}\left(\hat{b}_{i q}, \hat{b}_{i j}\right)$ between a pair of belief masses $\hat{b}_{i q}$ and $\hat{b}_{i j}$ is expressed by:

\begin{equation}
\operatorname{Bal}\left(\hat{b}_{i q}, \hat{b}_{i j}\right)=\left\{\begin{array}{cc}
1-\frac{\left|\hat{b}_{i q}-\hat{b}_{i j}\right|}{\hat{b}_{i q}+\hat{b}_{i j}} & \text { if } \hat{b}_{i q} \hat{b}_{i j} \neq 0, \\
0 & \text { otherwise }.
\end{array}\right.
\label{eq.12}
\end{equation}

To clarify the role of $\mathcal{L}_{Dis}$, we consider the three examples corresponding to (a), (b), and (c) in Figure~\ref{fig:Opinion_example}, respectively.

\begin{itemize}
    \item \textbf{Example~1:} $\hat{\boldsymbol{\alpha}}_1=(21, 2, 2)$, $\mathcal{L}_{Dis}\left(\hat{\boldsymbol{\alpha}}_1\right)=0.0873$.
    \item \textbf{Example~2:} $\hat{\boldsymbol{\alpha}}_2=(6, 6, 6)$, $\mathcal{L}_{Dis}\left(\hat{\boldsymbol{\alpha}}_2\right)=0.8333$.
    \item \textbf{Example~3:} $\hat{\boldsymbol{\alpha}}_3=(1.1, 1.1, 1.1)$, $\mathcal{L}_{Dis}\left(\hat{\boldsymbol{\alpha}}_3\right)=0.0909$.
\end{itemize}

Based on the three examples above, we find that $\hat{\boldsymbol{\alpha}}_1$ has the lowest dissonance because of the concentration of evidence and few conflicts. For the two conflict scenarios $\hat{\boldsymbol{\alpha}}_2$ and $\hat{\boldsymbol{\alpha}}_3$, the conflict is the largest in $\hat{\boldsymbol{\alpha}}_2$ because it has sufficient evidence for each class label, unlike $\hat{\boldsymbol{\alpha}}_3$ which has less evidence. In addition, for the OOD sample $\hat{\boldsymbol{\alpha}}_3$, the uncertainty $u=\frac{K}{S}$ mainly comes from the vacuity ($S$ is small) caused by the lack of evidence. In terms of $\mathcal{L}_{Dis}$, we can punish the prediction of the conflict, encourage the model to make a confident prediction, and avoid the adverse impact of the conflict scenario on the final result and uncertainty. 

With the guidance of the above two losses, \ourmethod is able to make categorical predictions with less risk of belief conflict. However, for $\hat{u}=1-\sum_{k=1}^K \hat{b}_k$, we consider the extreme case where $\hat{b}_t=0$ and $\sum_{k=1}^K \hat{b}_k=1$, where $t$ is the ground truth class. In this case, when more beliefs was generated in the incorrect labels of the sample, $\hat{u}$ will still be small, but then we should not allow the model to make such a ``false confident prediction''. Therefore, for the wrong class labels of the sample, we should ensure that the total evidence generated by the model for it is zero. To this end, we further incorporate a Kullback-Leibler (KL) divergence loss function, which adjusts the predicted distribution by penalizing those instances with evidence that deviates from the true data distribution. Specifically, the KL divergence loss is calculated as follows: 

\begin{equation}
\begin{aligned}
& \mathcal{L}_{KL}\left(\hat{\boldsymbol{\alpha}}_i\right)=K L\left[\operatorname{Dir}\left(\mathbf{p}_{i} \mid \tilde{\boldsymbol{\alpha}}_i\right)|| \operatorname{Dir}\left(\mathbf{p}_{i} \mid \mathbf{1}\right) \mid\right. \\
& =\log \left(\frac{\Gamma\left(\sum_{j=1}^K \tilde{\alpha}_{i j}\right)}{\Gamma(K) \prod_{j=1}^K \Gamma\left(\tilde{\alpha}_{i j}\right)}\right)+ \\
&\quad +
\sum_{j=1}^K\left(\tilde{\alpha}_{i j}-1\right)\left[\psi\left(\tilde{\alpha}_{i j}\right)-\psi\left(\sum_{j=1}^K \tilde{\alpha}_{i j}\right)\right],
\end{aligned}
\label{eq.13}
\end{equation}

\noindent where $\Gamma(\cdot)$ is the gamma function, and $\psi(\cdot)$ is the digamma function. Additionally, $\operatorname{Dir}\left(\mathbf{p}_i \mid \mathbf{1}\right)=\operatorname{Dir}\left(\mathbf{p}_i \mid\langle 1, \ldots, 1\rangle\right)$ is the uniform Dirichlet distribution, and $\tilde{\boldsymbol{\alpha}}_i=\mathbf{Y}_i+\left(1-\mathbf{Y}_i\right) \odot \hat{\boldsymbol{\alpha}}_i$ is the Dirichlet parameters after removal of the non-misleading evidence from predicted parameters $\hat{\boldsymbol{\alpha}}_i$ for node~$v_i$. With $\mathcal{L}_{KL}$, the model avoids producing excessive evidence for the wrong class label of the sample, which guarantees that the uncertainty quantification is justified.

To sum up, the total objective function of \ourmethod can be obtained by adding three loss terms together: 

\begin{equation}
\mathcal{L}(\hat{\boldsymbol{\alpha}})=\sum_{i=1}^n \mathcal{L}_{ECE}\left(\hat{\boldsymbol{\alpha}}_i\right)+\lambda_{Dis} \mathcal{L}_{Dis}\left(\hat{\boldsymbol{\alpha}}_i\right)+\lambda_{KL} \mathcal{L}_{KL}\left(\hat{\boldsymbol{\alpha}}_i\right),
\label{eq.14}
\end{equation}

\noindent where $\lambda_{KL}, \lambda_{Dis}\geq0$ is the balance factor. The learning objective allows \ourmethod to explore the parameter space while focusing on penalty the evidence corresponding to incorrect classes and the minimization of dissonance in the opinion. The training algorithm of \ourmethod is summarized in Algorithm 1.

\begin{table}
\centering
\begin{tabular}{l} 
\toprule
\textbf{Algorithm 1. }Algorithm for \ourmethod. \\ 
\hline
\begin{tabular}[c]{@{}l@{}}\textbf{Input: }Adjacency matrix $\mathbf{A}$; Feature matrix $\mathbf{X}$; Propagation\\~iteration: $T$; MLP model: $f_{\Theta}^e$; Perturbation probability $\mathcal{b}$.\end{tabular} \\ 
\textbf{Output: }Prediction $\mathbf{Z}$  and its uncertainty $\mathbf{u}$. \\ 
\textbf{Initialize:} $\mathbf{X}^0=\mathbf{X}, \tilde{\mathbf{A}}=\mathbf{A}+\mathbf{I}, \hat{\mathbf{A}}=\widetilde{\mathbf{D}}^{-\frac{1}{2}} \tilde{\mathbf{A}} \widetilde{\mathbf{D}}^{-\frac{1}{2}}$.\\ 
\textbf{Operation 1: Embedding Propagation.} \\ 
\begin{tabular}[c]{@{}l@{}}1 ~\textbf{for} $t$ = 1:$T$ \textbf{do} \\2 ~~~~~Propagation: $\mathbf{X}^t=\hat{\mathbf{A}} \mathbf{X}^{t-1}$\end{tabular} \\ 
\textbf{Operation 2: Evidential Learning.} \\ 
3 ~\textbf{while} not convergence \textbf{do} \\ 
4~~~ ~~\textbf{for} $t$ = 1:$T$ \textbf{do} \\ 
5~~~ ~~~~~~Random perturbations: $\tilde{\mathbf{X}}^t \sim \operatorname{Dropout}\left(\mathbf{X}^t, \mathcal{b}\right)$\\ 
6 ~~~~~~~~~Obtain evidence using shared-MLP: $\mathbf{e}^t=f_{\Theta}^e\left(\tilde{\mathbf{X}}^t\right)$ \\ 
7~~~ ~~~~~~Obtain Dirichlet distribution $\operatorname{Dir}\left(\mathbf{p} \mid \boldsymbol{\alpha}^t\right), \boldsymbol{\alpha}^t=\mathbf{e}^t+1$ \\ 
8 ~~~~~~~~~Obtain multinomial opinion $\boldsymbol{\omega}^t$~with Eq.~\ref{eq.3}. \\
9~~~ ~~\textbf{end for} \\
10~~~~Obtain joint opinion $\boldsymbol{\Omega}$~and $\operatorname{Dir}(\mathbf{p} \mid \hat{\boldsymbol{\alpha}})$ with Eq.~\ref{eq.8},~\ref{eq.9}. \\
11~~~~Calculate overall loss $\mathcal{L}(\hat{\boldsymbol{\alpha}})$~with Eq.~\ref{eq.14}. \\
12~~~~Update parameters via gradient descent. \\
13 \textbf{end while} \\
14 Output $\mathbf{Z}$ and $\mathbf{u}$ with Eq.~\ref{eq.9}. \\
\bottomrule
\end{tabular}
\end{table}\label{Algorithm}

\subsection{Theoretical Analysis of the CBF}\label{VD}
In order to prove the validity of CBF integrating GNNs evidence under different neighborhoods, in this subsection, we provide theoretical analysis from two perspectives of uncertainty and accuracy respectively.
Specifically, CBF satisfies the same four propositions as described in~\cite{han2022trusted}. In contrast, the fusion of Dempster's rule from~\cite{han2022trusted} produces counter-intuitive results~\cite{zadeh1984review}. In addition, the fusion rule from~\cite{han2022trusted} is generally not applicable to fusions involving completely contradictory opinions~\cite{josang2016subjective}. However, nodes can obtain conflicting information from different neighborhoods. Furthermore, in $K$-classified opinion fusion with $n$ nodes and $L$ neighbors, the fusion rule proposed in~\cite{han2022trusted} exhibits a time complexity of $\mathcal{O}(LnK^2)$, whereas CBF demonstrates a time complexity of just $\mathcal{O}(nK)$, regardless of the number of views. Thus, CBF emerges as a more suitable choice for multi-hop neighborhood evidence fusion in GNNs, all while adhering to the same propositions outlined in~\cite{han2022trusted}.

In particular, for the respective opinions $\boldsymbol{\omega}^\xi$ and $\boldsymbol{\omega}^\varphi$ of the same node under views $\xi$ and $\varphi$, when $u^{\xi} \neq 0 \vee u^{\varphi} \neq 0$ we have:

\begin{equation}
\hat{u}=\frac{u^{\xi} u^{\varphi}}{u^{\xi}+u^{\varphi}-u^{\xi} u^{\varphi}}, \quad \hat{b}_t=\frac{b_t^{\xi} u^{\varphi}+u^{\xi} b_t^{\varphi}}{u^{\xi}+u^{\varphi}-u^{\xi} u^{\varphi}}.
\label{eq.15}
\end{equation}

On the basis of the definitions, we can deduce Propositions~\ref{proposition1} and~\ref{proposition2} from the perspective of uncertainty. 

\begin{proposition}
\label{proposition1}
Assuming $e^{\varphi}$ is the evidence of the sample generated by the view $\varphi$ and $\hat{e}$ is evidence obtained by CBF, the uncertainty $\hat{u}$ after fusion is smaller than the uncertainty $u^\varphi$ in view $\varphi$. 
\end{proposition}
\begin{proof}
\begin{equation}
\begin{aligned}
\hat{u} & =\frac{u^{\xi} u^{\varphi}}{u^{\xi}+u^{\varphi}-u^{\xi} u^{\varphi}} \\
& \leq \frac{u^{\xi} u^{\varphi}}{u^{\xi}} \\
& = u^{\varphi}.
\end{aligned}
\label{eq.16}
\end{equation}

\end{proof}

\begin{proposition}
\label{proposition2}
For any two views $\xi$ and $\varphi$, the fusion uncertainty $\hat{u}$ is positively correlated with the single uncertainty $u^\varphi$ and $u^\xi$. 
\end{proposition}
\begin{proof}
\begin{equation}
\begin{aligned}
\hat{u} & =\frac{u^{\xi} u^{\varphi}}{u^{\xi}+u^{\varphi}-u^{\xi} u^{\varphi}} \\
& =\frac{1}{\frac{1}{u^{\xi}}+\frac{1}{u^{\varphi}}-1}.
\end{aligned}
\label{eq.17}
\end{equation}
\end{proof}

Proposition~\ref{proposition1} reveals that after fusing evidence from multiple neighborhoods, the uncertainty of the final prediction will be less than or equal to the uncertainty of any single neighborhood. However, Proposition~\ref{proposition2} shows that when the prediction uncertainty of the evidence of the two neighborhoods involved in fusion is large, the prediction uncertainty after fusion will also be large, which is in line with intuition.

Moreover, from the perspective of accuracy, we can deduce  Propositions~\ref{proposition3} and \ref{proposition4}.

\begin{proposition}
\label{proposition3}
Under the conditions $b_t^{\xi} \geq b_m^{\varphi}$, where $t$ is the index of ground-truth class and $b_m^{\varphi}$ is the largest belief mass in view $\varphi$. Integrating another opinion $\omega^\xi$ makes the joint opinion satisfy $\hat{b}_t \geq b_t^{\varphi}$.
\end{proposition}
\begin{proof}
\begin{equation}
\begin{aligned}
\hat{b}_t & =\frac{b_t^{\xi} u^{\varphi}+u^{\xi} b_t^{\varphi}}{u^{\xi}+u^{\varphi}-u^{\xi} u^{\varphi}} \\
& \geq \frac{b_t^{\varphi}\left(u^{\varphi}+u^{\xi}\right)}{u^{\xi}+u^{\varphi}-u^{\xi} u^{\varphi}} \\
& \geq \frac{b_t^{\varphi}\left(u^{\varphi}+u^{\xi}\right)}{u^{\xi}+u^{\varphi}}=b_t^{\varphi}.
\end{aligned}
\label{eq.18}
\end{equation}
\end{proof}

\begin{proposition}
\label{proposition4}
When $u^\varphi$ is small, $b_t^{\varphi}-\hat{b}_t$ will be limited, and it will have a positive correlation with $u^\varphi$. As a special case, when $u^\varphi$ is small enough (i.e., $u^\varphi$=0), integrating another opinion will not reduce the performance (i.e.,$b_t^{\varphi}=\hat{b}_t$). 
\end{proposition}
\begin{proof}
\begin{equation}
\begin{aligned}
b_t^{\varphi}-\hat{b}_t & =b_t^{\varphi}-\frac{b_t^{\xi} u^{\varphi}+u^{\xi} b_t^{\varphi}}{u^{\xi}+u^{\varphi}-u^{\xi} u^{\varphi}} \\
& \leq b_t^{\varphi}-\frac{u^{\xi} b_t^{\varphi}}{u^{\xi}+u^{\varphi}-u^{\xi} u^{\varphi}} \\
& \leq b_t^{\varphi}-\frac{u^{\xi} b_t^{\varphi}}{u^{\xi}+u^{\varphi}} \\
& =b_t^{\varphi}\left(\frac{1}{\frac{u^{\xi}}{u^{\varphi}}+1}\right).
\end{aligned}
\label{eq.19}
\end{equation}
\end{proof}

For the benefits of fusion, Proposition~\ref{proposition3} reveals that fusing evidence from different neighborhoods can potentially improve the accuracy of the model. For the negative effects that fusing different neighborhood information can have on some nodes, Proposition~\ref{proposition4} ensures under mild conditions, the model is limited in terms of the performance degradation that can occur when another neighborhood evidence is incorporated into the original neighborhood.

In conclusion, \ourmethod distinctly demonstrates theoretical advantages in both uncertainty and accuracy.

\subsection{Model Complexity Analysis}\label{VE}
In this subsection, we discuss the time complexity of \ourmethod. The complexity of \ourmethod is mainly composed of EP and ET. The complexity of EP is $\mathcal{O}(Ld|E|)$, where $L$ denotes propagation step, $d$ is the dimension of node feature, and $|E|$ denotes edge count. The complexity of its ET (i.e., two-layer shared-weights MLP) is $\mathcal{O}\left(n d_h (d+K)\right)$, where $d_h$ refers to its hidden layer size, $n$ is the number of nodes, and $K$ is the number of classes. To sum up, the total computational complexity of \ourmethod is $\mathcal{O}\left(L d (|E|)+n d_h (d+K)\right)$, which is linear with the sum of node number and edge number.

%% file: 5_Exp.tex
\begin{table}
\centering
\caption{The statistics of datasets.}
\resizebox{1\columnwidth}{!}{
\begin{tabular}{l|ccccc} 
\toprule
\textbf{Dataset} & \textbf{Nodes} & \textbf{Features} & \textbf{Edges} & \textbf{Classes} & \textbf{Train/Val/Test} \\ 
\midrule
\textbf{Cora} & 2708 & 1433 & 5278 & 7 & 140/500/1000 \\ 
\textbf{Citeseer} & 3327 & 3703 & 4552 & 6 & 120/500/1000 \\ 
\textbf{Pubmed} & 19717 & 500 & 44324 & 3 & 60/500/1000 \\ 
\textbf{Computers} & 13381 & 767 & 245778 & 10 & 200/300/12881 \\ 
\textbf{Photo} & 7487 & 745 & 119043 & 8 & 160/240/7087 \\
\textbf{Actor} & 7600 & 932 & 30019 & 5 & 100/150/7350 \\
\textbf{Chameleon} & 2277 & 2325 & 36101 & 5 & 100/150/2027 \\
\textbf{Squirrel} & 5201 & 2089 & 217073 & 5 & 100/150/4951 \\
\bottomrule
\end{tabular}
}
\label{t1}
\end{table}

In this section, we evaluate \ourmethod with extensive experiments. We first describe the experimental setup, including datasets, baseline methods, and implementation details. Then, we present the experimental results on performance comparison and analytical experiments. 

\subsection{Experimental Setup}

\subsubsection{Datasets} 
To demonstrate the effectiveness of our model, we conducted comparative experiments on eight datasets. These include homophilic graphs, specifically, three citation networks (Cora, Citeseer, and Pubmed) and two co-purchase networks (Amazon Computers and Amazon Photo)~\cite{zhang2022graph}, as well as heterophilic graphs (Actor, Chameleon, and Squirrel)~\cite{li2024permutation}. We applied standard fixed dataset splits on all datasets, following the evaluation protocol in~\cite{zhang2022graph}. The detailed information about the datasets we adopted during the experiments is in Table~\ref{t1}.
\input{Tables/main_table}

\subsubsection{Baselines} We compared \ourmethod with ten state-of-the-art GNNs, including GCN~\cite{kipf2016semi}, GAT~\cite{velivckovic2017graph}, JK-Net~\cite{xu2018representation}, ResGCN~\cite{li2019deepgcns}, APPNP~\cite{gasteiger2018predict}, AP-GCN~\cite{spinelli2020adaptive}, SGC~\cite{wu2019simplifying}, SIGN~\cite{frasca2020sign}, $\mathrm{S^{2}GC}$~\cite{zhu2020simple}, GAMLP~\cite{zhang2022graph}, AGNN~\cite{chen2023agnn}, and Flip-APPNP~\cite{choi2024beyond}.
We also consider four uncertainty-aware GNNs, Drop-GAT~\cite{gal2016dropout,zhao2020uncertainty,ryu2019uncertainty}, S-BGAT-T~\cite{zhao2020uncertainty}, GKDE~\cite{zhao2020uncertainty}, and CaGCN~\cite{wang2021confident}. 
Among them, Drop-GAT and S-BGAT-T (reference~\cite{zhao2020uncertainty} Supplementary material Table 8) have superior performance compared to GCN-based Drop-GCN and S-BGCN-T. 
In Drop-GAT, Monte-Carlo Dropout is incorporated into the GAT model to learn probabilistic uncertainty. S-BGAT-T is a subjective logic GAT based on Bayesian framework with teacher network. GKDE is a graph-based kernel Dirichlet distribution estimation method. CaGCN is a GCN model based on confidence calibration. For the above baselines, reference was made to the experimental results reported in paper~\cite{zhang2022graph,zhao2020uncertainty}. During reproduction, the experiments are conducted based on the hyperparameters provided in the original papers.

\subsubsection{Implementation Details} 
For each experiment, we reported an average performance of 10 runs. The hyperparameters of all models are adjusted based on the loss and accuracy on the validation set. For hyperparameter settings, we perform a specific set of grid searches to select the key hyperparameters in \ourmethod. The search space is as follows: 
\begin{itemize}
    \item Learning rate: \{5e-3, 1e-2, 2e-2\}
    \item Weight decay: \{1e-3, 5e-3, 1e-2, 2e-2\}
    \item Hidden size: \{32, 64, 256\}
    \item Dropout rate: \{0., 0.2, 0.4, 0.6, 0.8\}
    \item Perturbation probability $\mathcal{b}$: \{0.1, 0.3, 0.5, 0.7\}
    \item Propagation iterations $T$: \{2, 4, 6, 8, 16, 32\}
    \item $\lambda_{KL}$: \{0.01, 0.05, 0.1\}
    \item  $\lambda_{Dis}$: \{0.1, 0.3, 0.5\}
\end{itemize}

The experiments are conducted on a machine with Intel(R) Xeon(R) Platinum 8255C CPU@2.50GHz and a single RTX 3090 GPU with 24GB GPU memory. The operating system of the machine is Ubuntu 20.04. As for software versions, we use Python 3.8, Pytorch 2.1.2, and CUDA 12.1.

\subsection{Performance of Node Classification}
In this experiment, we compared the node classification performance of \ourmethod and baseline GNNs. As can be seen from Table~\ref{t2}, \ourmethod achieved state-of-the-art performance on two citation datasets and second-best accuracy on the remaining one. Moreover, it also outperformed other methods on two Amazon datasets. Compared to SGC, SIGN, and $\mathrm{S^{2}GC}$ which have the same decoupled architecture as \ourmethod ($EP \to EP \cdots ET \to ET$), \ourmethod improved the accuracy by 2.4\% to 5.2\%. This indicates the benefits of \ourmethod within the same GNN architecture. \ourmethod also surpassed uncertainty-aware GNNs, which consider the problem of trustworthy prediction by GNNs. For CaGCN, \ourmethod leads by 4\% on the Pubmed datasets. These results show the remarkable effectiveness of \ourmethod based on evidence learning and evidence fusion in handling graph data with uncertainty.

\begin{figure*} [!t]
	\centering
	\subfloat[Cora\label{fig8:a}]{
		\includegraphics[scale=0.18]{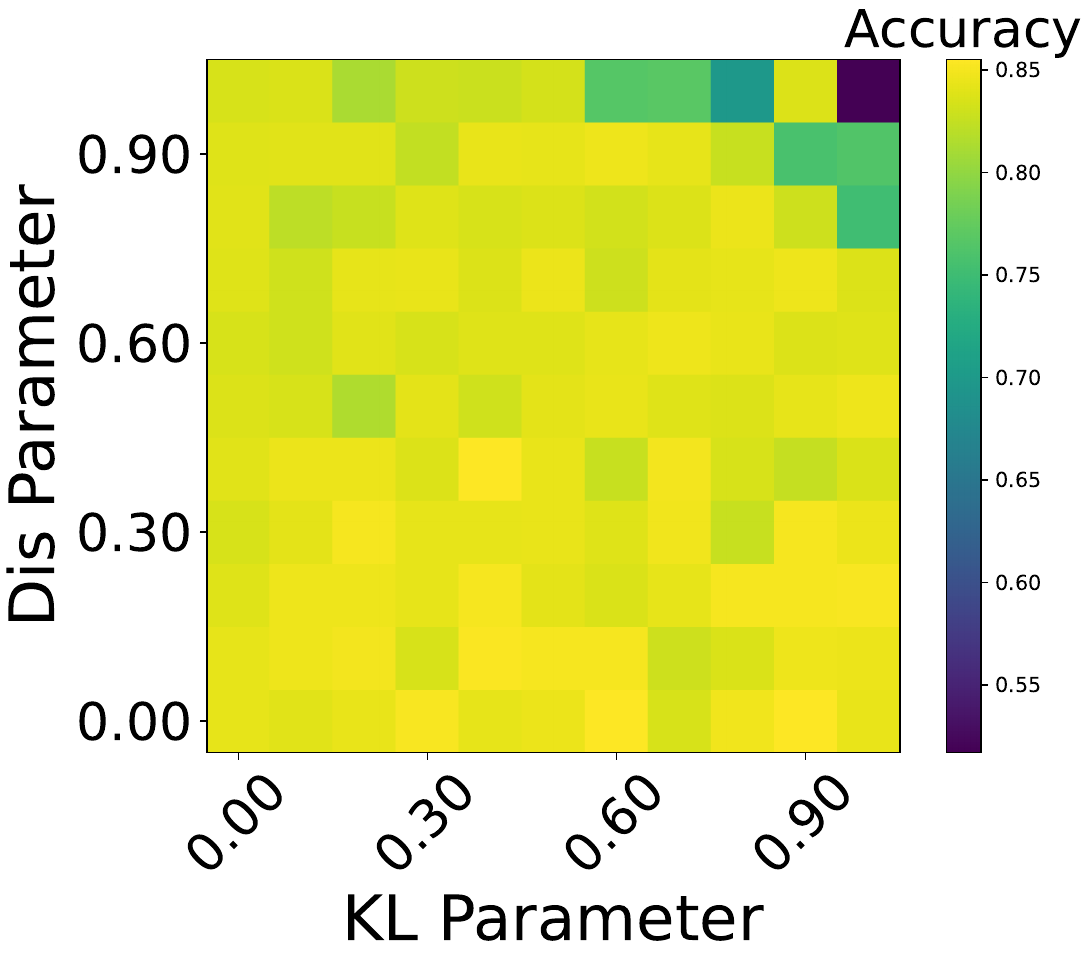}}\hfill
	\subfloat[Citeseer\label{fig8:b}]{
		\includegraphics[scale=0.18]{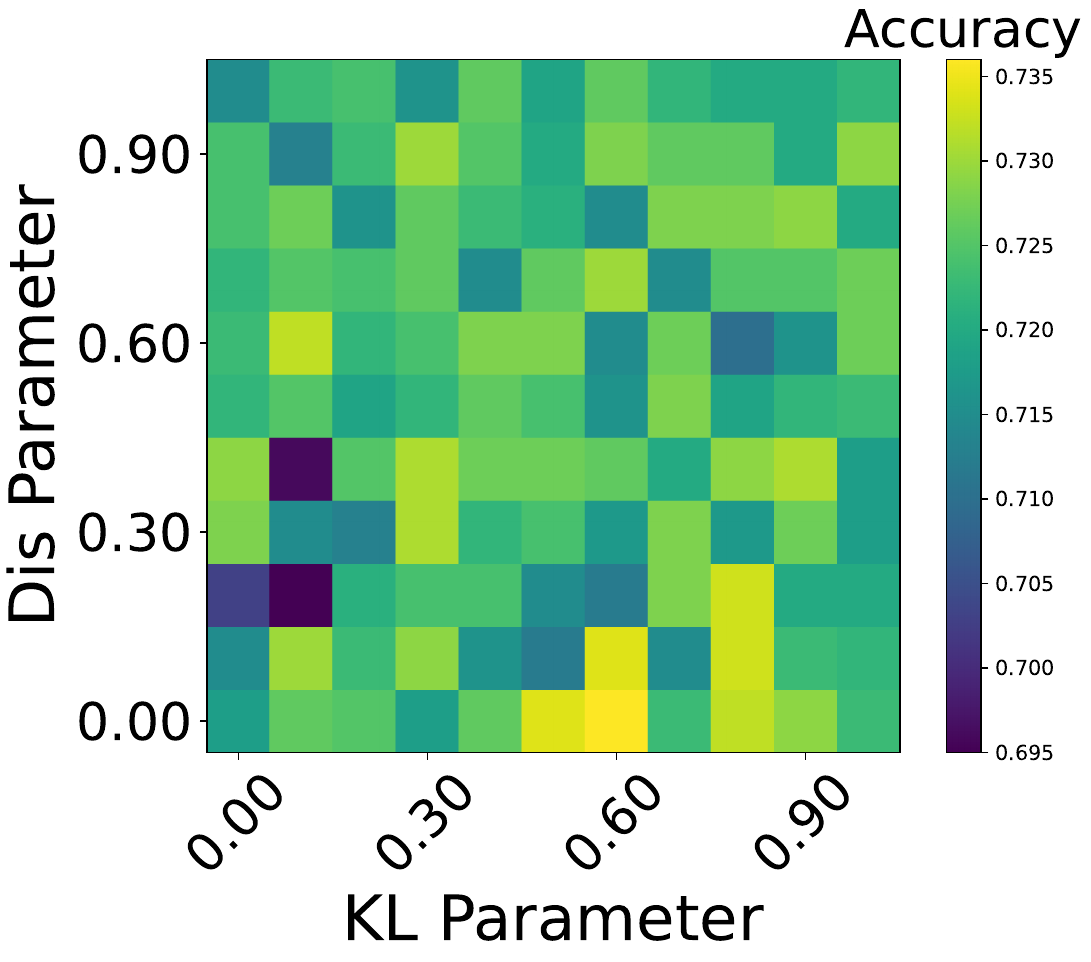}}\hfill
	\subfloat[Pubmed\label{fig8:c}]{
		\includegraphics[scale=0.18]{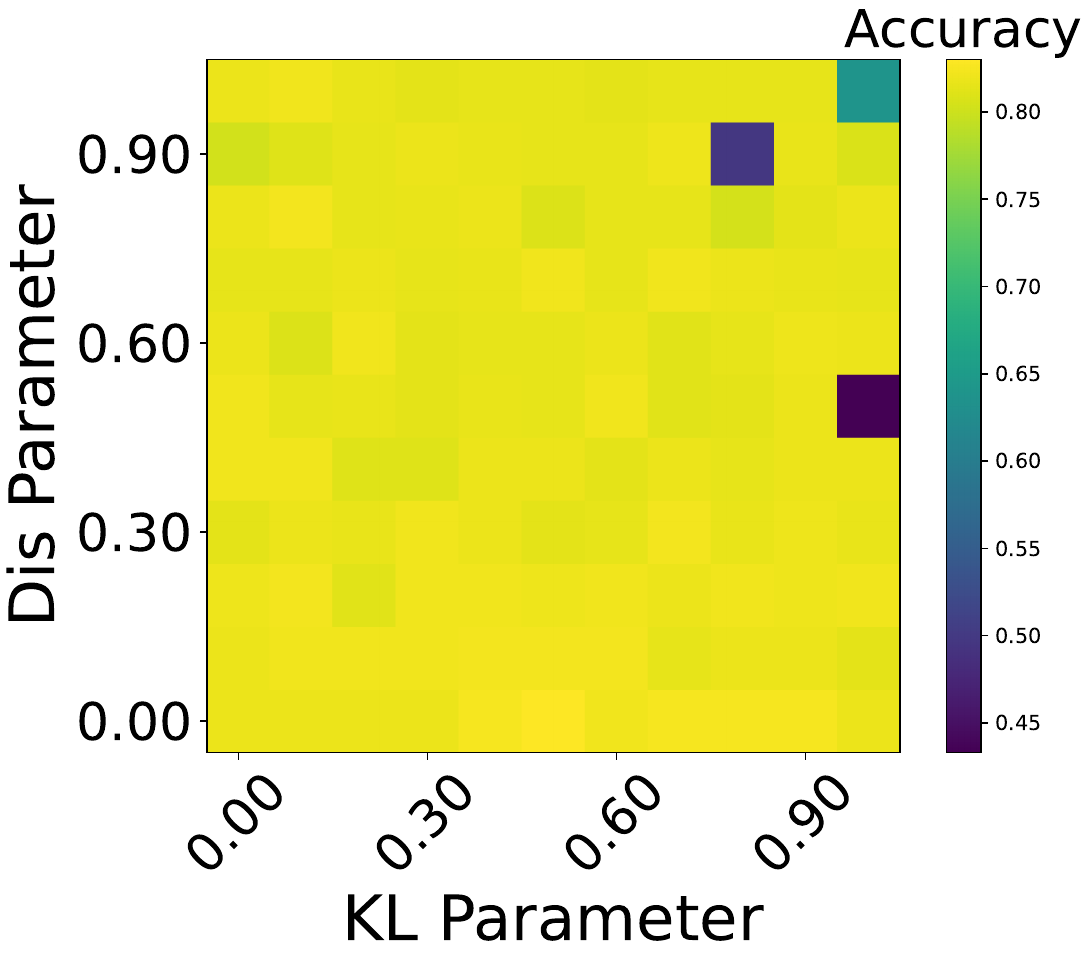}}\hfill
	\subfloat[Computers\label{fig8:d}]{
		\includegraphics[scale=0.18]{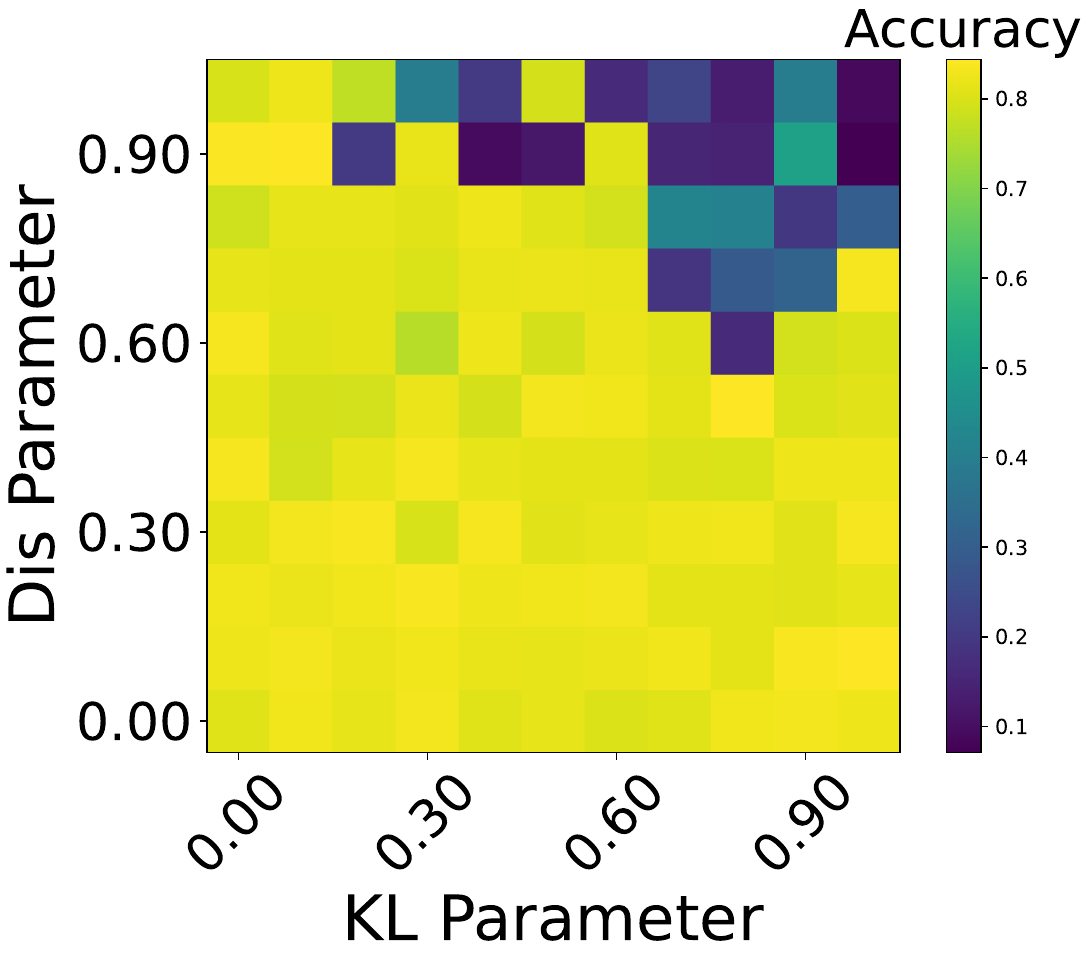}}\hfill
	\subfloat[Photo\label{fig8:e}]{
		\includegraphics[scale=0.18]{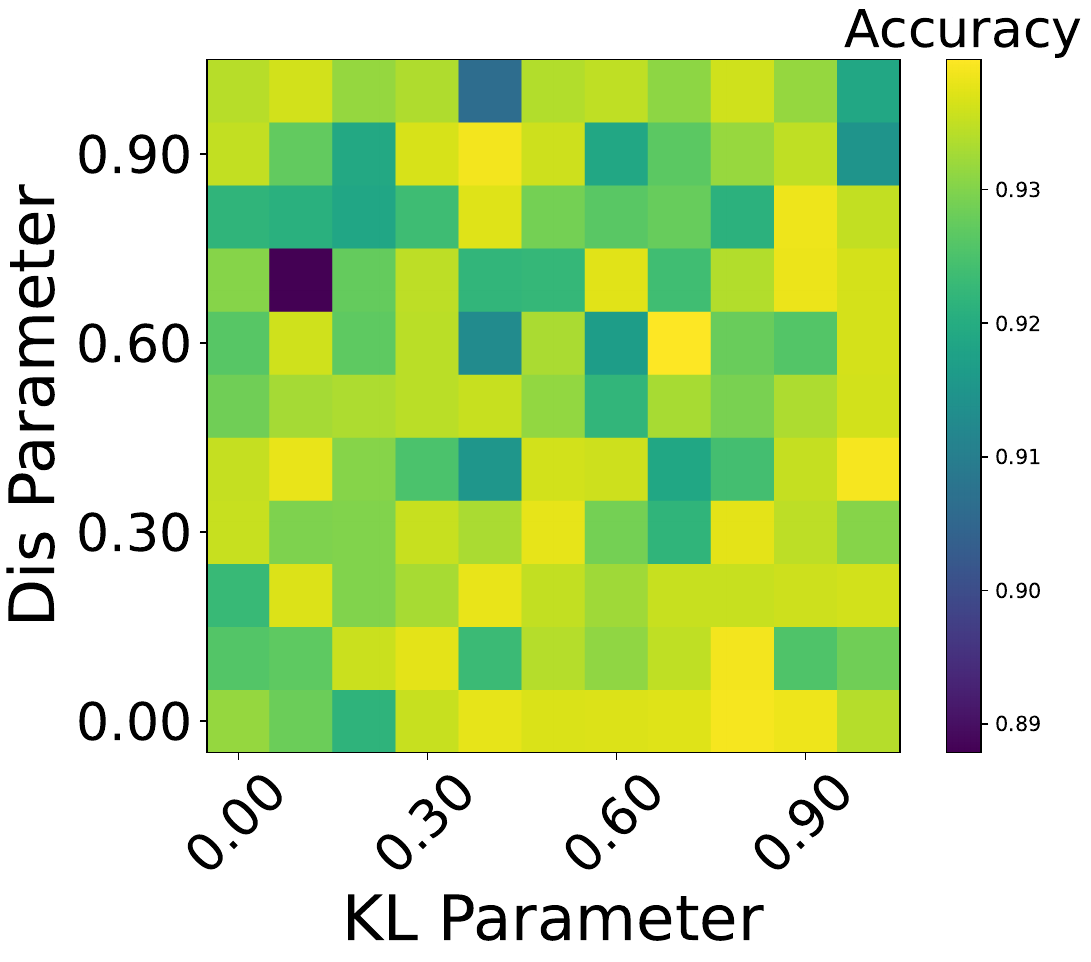}}
	\caption{Heat map visualization of the effects of two parameters $\lambda_{KL}$ and $\lambda_{Dis}$ on model performance.}
	\label{fig:param_heatmap}
\end{figure*}

\begin{figure*} [!t]
	\centering
	\subfloat[Cora\label{fig9:a}]{
		\includegraphics[scale=0.18]{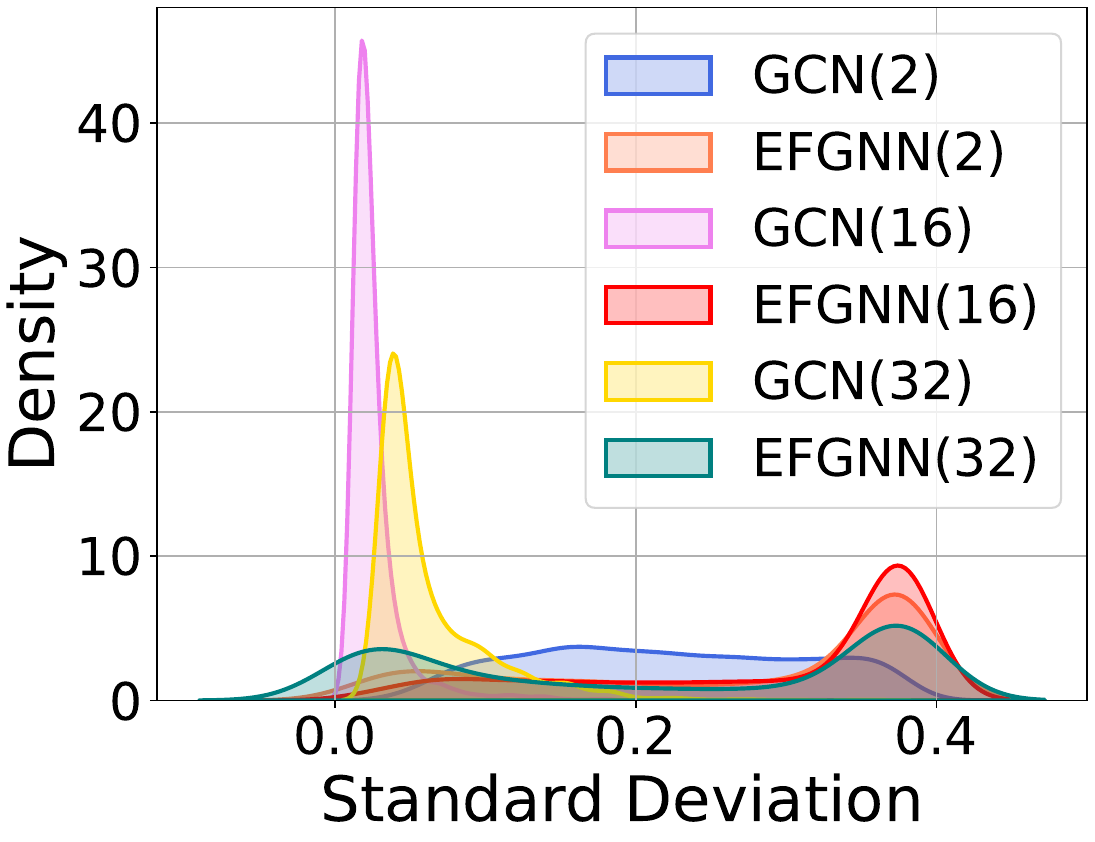}}\hfill
	\subfloat[Citeseer\label{fig9:b}]{
		\includegraphics[scale=0.18]{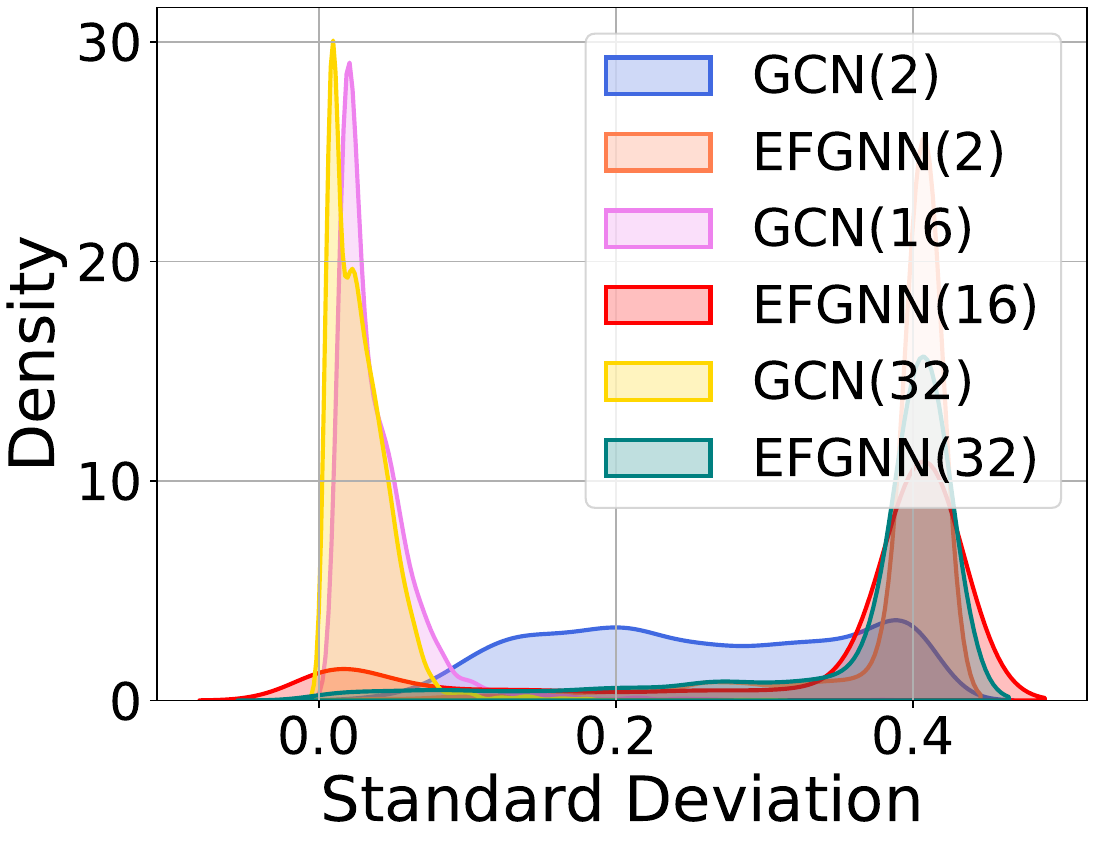}}\hfill
	\subfloat[Pubmed\label{fig9:c}]{
		\includegraphics[scale=0.18]{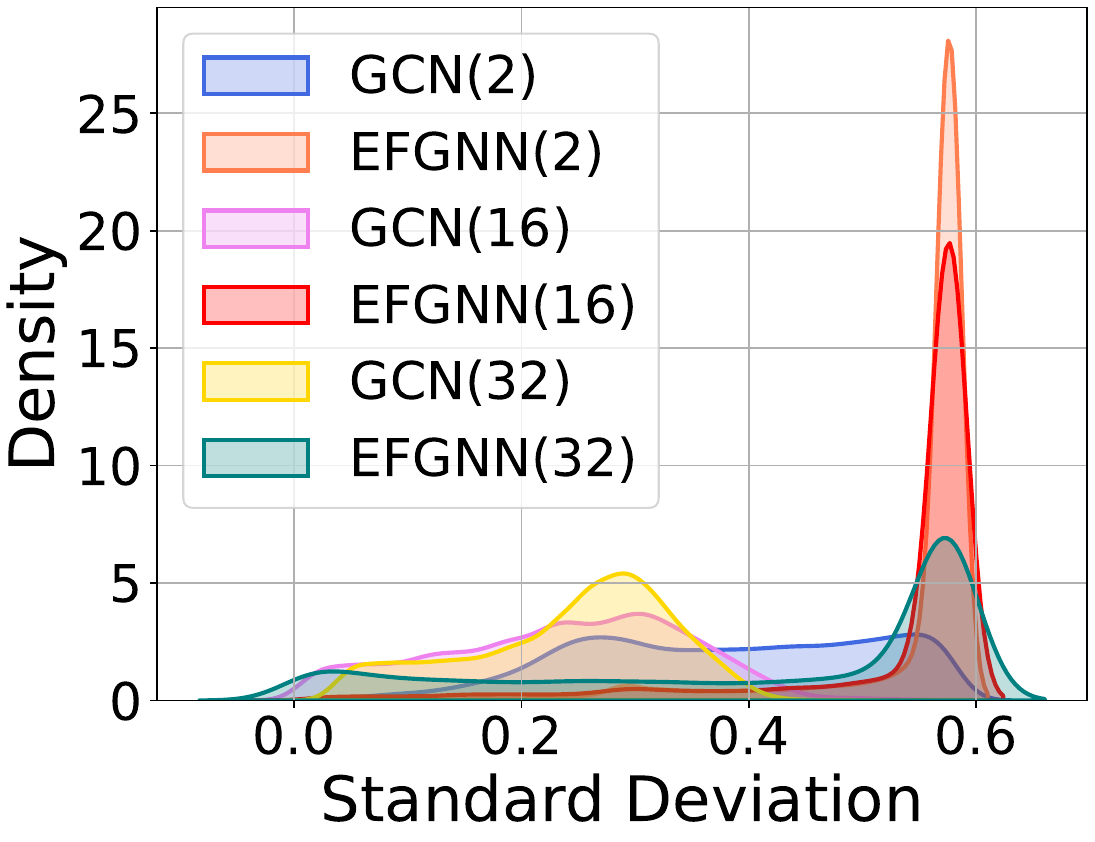}}\hfill
	\subfloat[Computers\label{fig9:d}]{
		\includegraphics[scale=0.18]{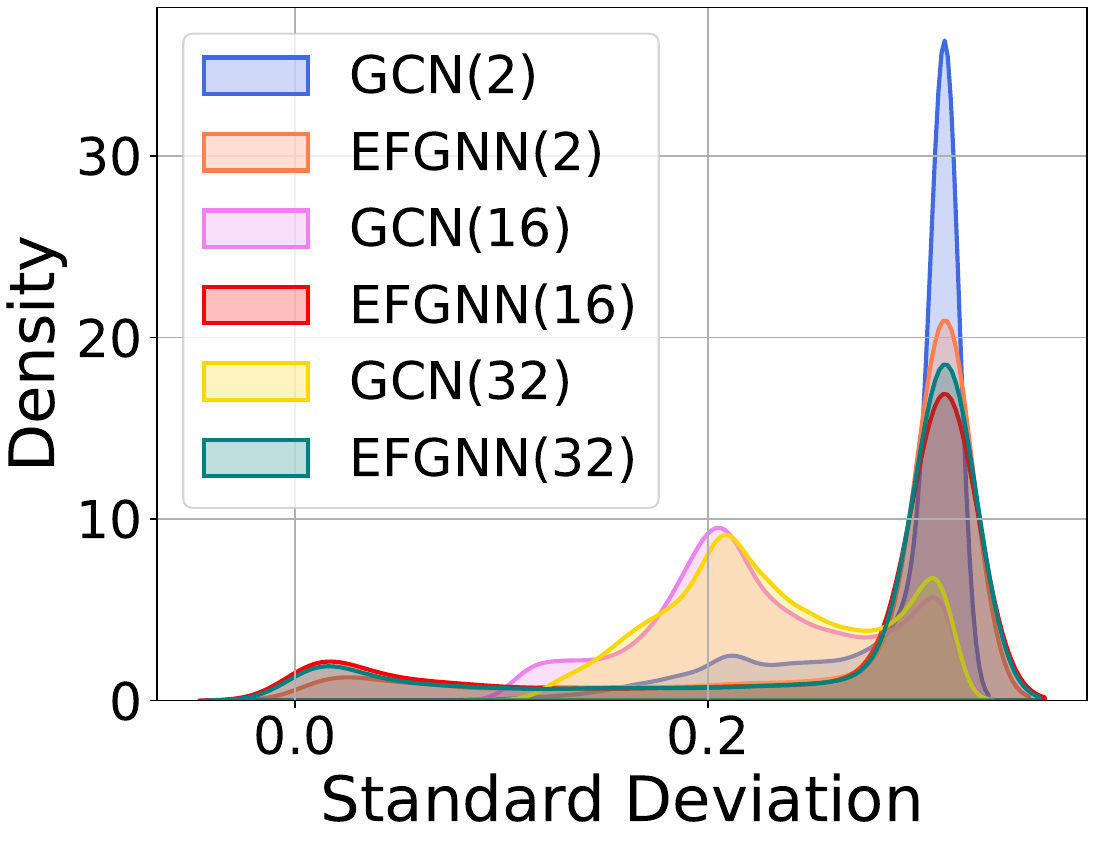}}\hfill
	\subfloat[Photo\label{fig9:e}]{
		\includegraphics[scale=0.18]{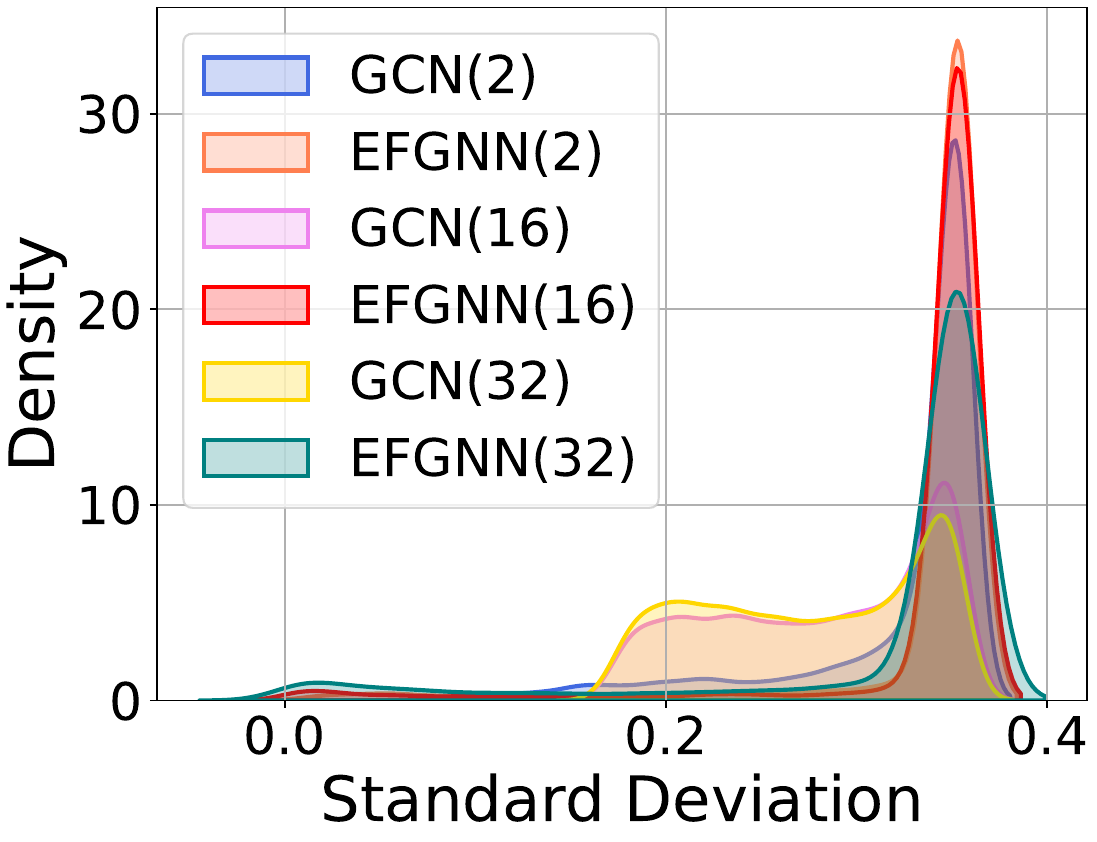}}
	\caption{The standard deviation of the class probability distribution of multi-layer (2, 16, or 32 layers) GCN and \ourmethod.}
	\label{fig:std_dis_compare}
\end{figure*}

\subsection{Ablation Study}
In order to analyze the contribution of each component and key designs in \ourmethod, we perform experiments on several variants of \ourmethod. We first remove the regularization terms by setting $\lambda_{KL}$ and $\lambda_{Dis}$ to be 0, denoted as ``w/o KL'' and ``w/o Dis'', respectively. Also, we create a variant without CBF (denoted as ``w/o CBF'') that uses the final layer output for prediction. The performance comparison is demonstrated in Table~\ref{t3}. 
From the results, it is easy to see that the performance degradation of ``w/o CBF'' is the largest among several variants, indicating the significance of evidence fusion. In addition, both KL divergence and dissonance regularization terms have positive effects on model performance. 

Furthermore, we set $\lambda_{KL}$ and $\lambda_{Dis}$ to values between 0-1 and visualized their effects in the form of heatmaps. As we can see from Figure~\ref{fig:param_heatmap}, when $\lambda_{KL}$ and $\lambda_{Dis}$ are too large, the accuracy of the model is low. This shows that focusing on the penalty for wrong classes and dissonance of evidence can cause the model to fail to find the optimal parameter space. Moreover, when the values of $\lambda_{KL}$ and $\lambda_{Dis}$ are mediate, the performance of \ourmethod is generally better.

In addition, Table~\ref{t4} provides an example of an 8-layer \ourmethod, with EP-$i$ representing only the evidence generated by the features after the $i$ aggregation as the output. We can witness that the performance of the complete \ourmethod is better than the performance of evidence classification using only any single output. Meanwhile, from EP-1 to EP-8, the optimal performance is reached at different depths for different datasets, which verifies the importance of considering multiple depths.

\input{Tables/ablation}

\subsection{Model Prediction Uncertainty Varies of Depth}

To verify the effectiveness of \ourmethod against the performance degradation problem of deep GNNs, we set the depths of 2, 8, 16, 32 and 64 layers respectively and conducted experiments with the GCNs with corresponding layers on five datasets. The results in Table~\ref{t5} demonstrate that the optimal propagation steps of \ourmethod are correlated with graph density. For relatively dense datasets, the optimal range typically lies between 4 to 8 hops, whereas for sparse datasets, it tends to span 10 to 16 hops.
Moreover, as the number of layers increases, the performance of GCN experiences a substantial decline. Conversely, the performance of EFGNN remains relatively stable over a wide range of propagation steps. This underscores that \ourmethod, built upon dependable multi-scale evidence fusion, effectively mitigates the degradation issue that plagues conventional GNNs as layer count increases. The distinct difference highlights the effectiveness of \ourmethod in leveraging long-range dependency to make predictions.

\input{Tables/depth}

\begin{figure} [tbp]
	\centering
 \subfloat[Citeseer]{
		\includegraphics[scale=0.17]{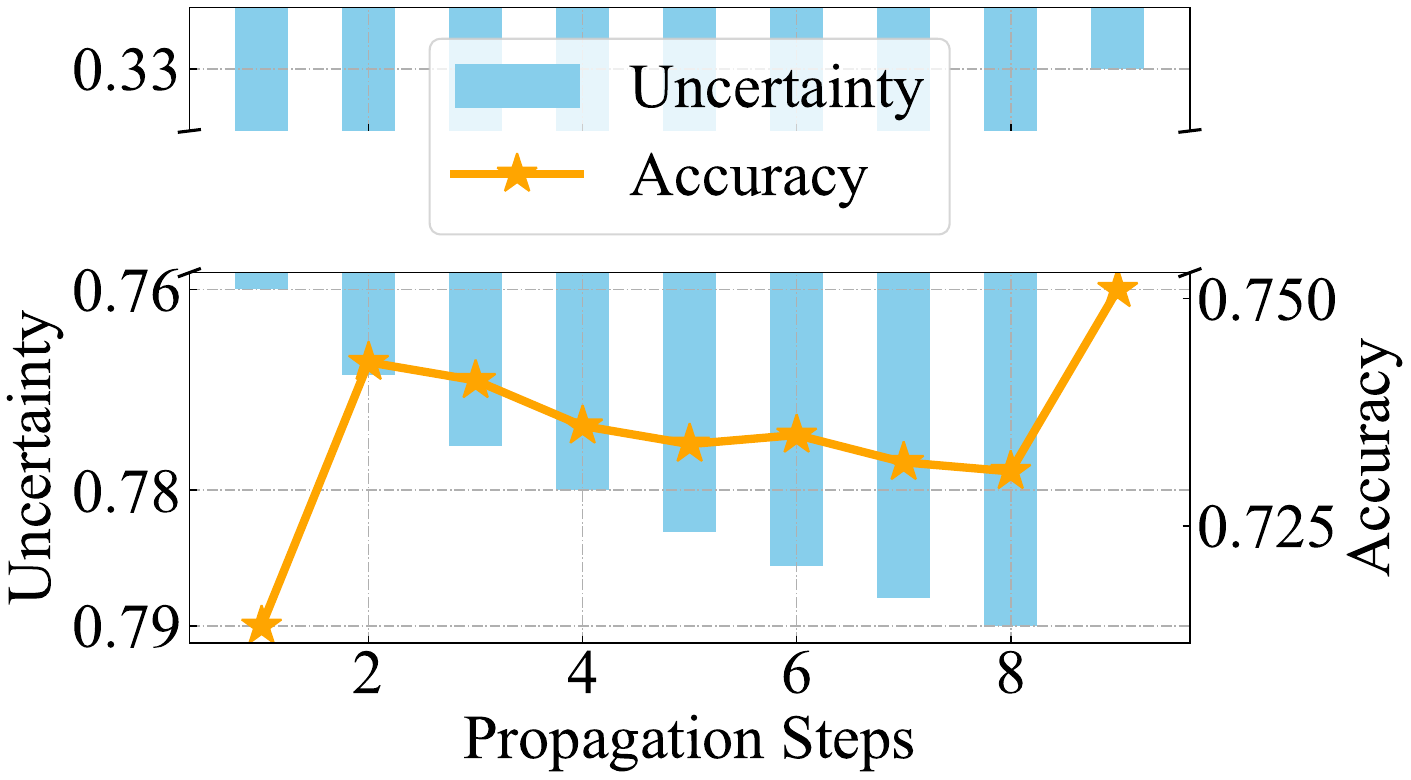}}\hfill
	\subfloat[Squirrel]{
		\includegraphics[scale=0.17]{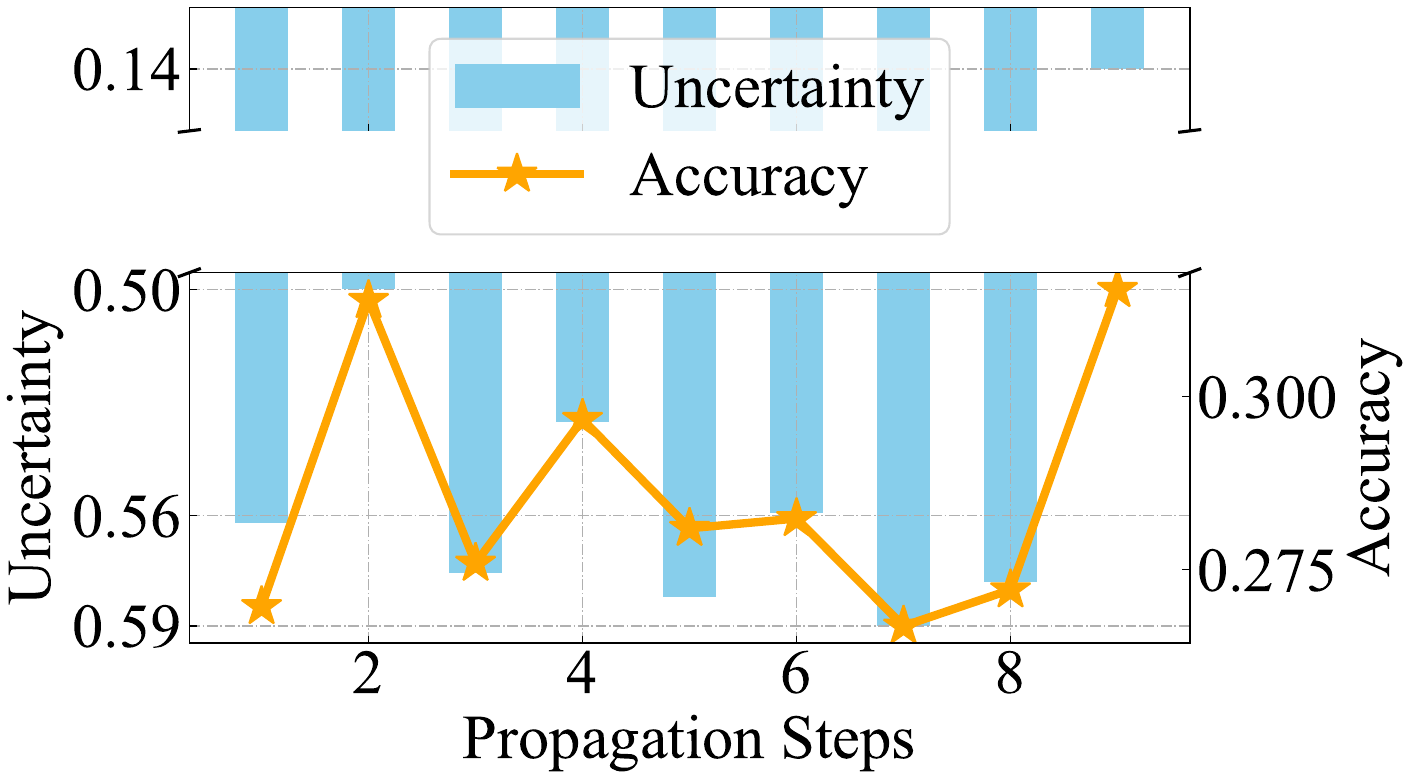}}
	\caption{The accuracy and uncertainty of \ourmethod in different propagation steps.}
    \label{fig:hop&ua}
\end{figure}

\begin{figure*} [tbp]
	\centering
 \subfloat[Cora\label{subfig:acc_u_cora}]{
		\includegraphics[scale=0.17]{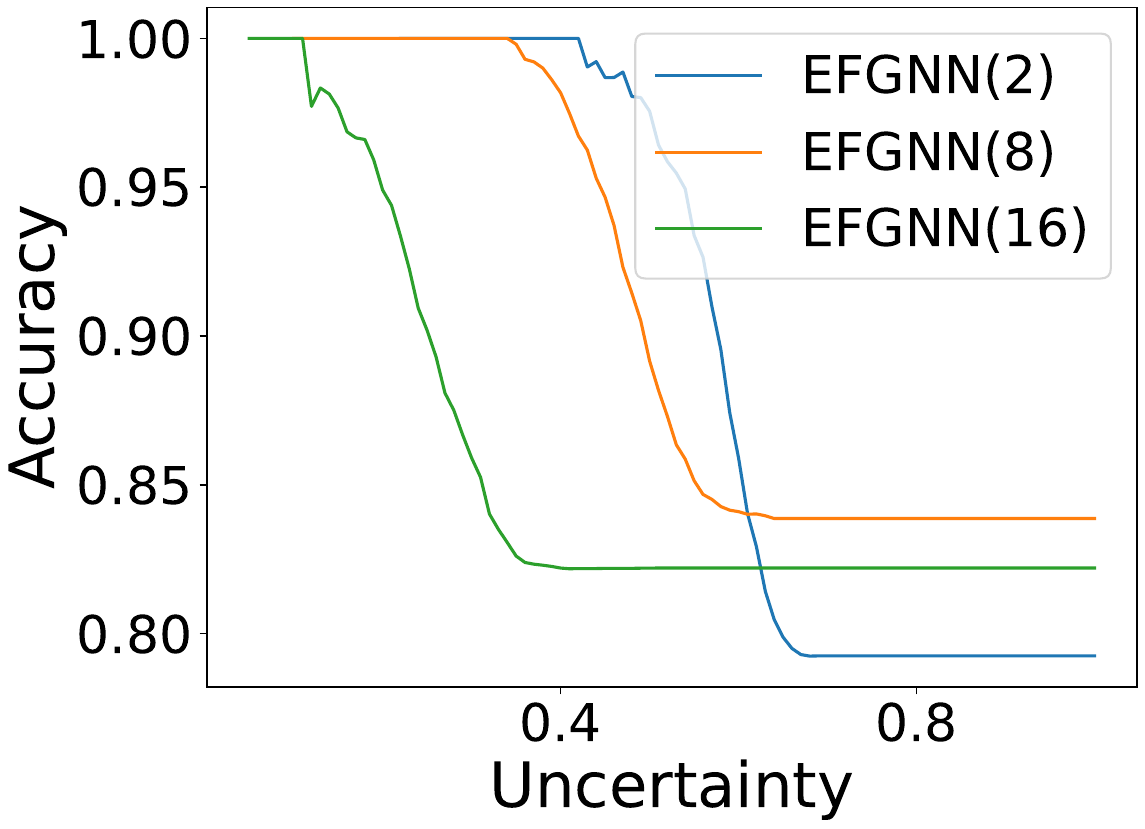}}\hfill
	\subfloat[Citeseer\label{subfig:acc_u_citeseer}]{
		\includegraphics[scale=0.17]{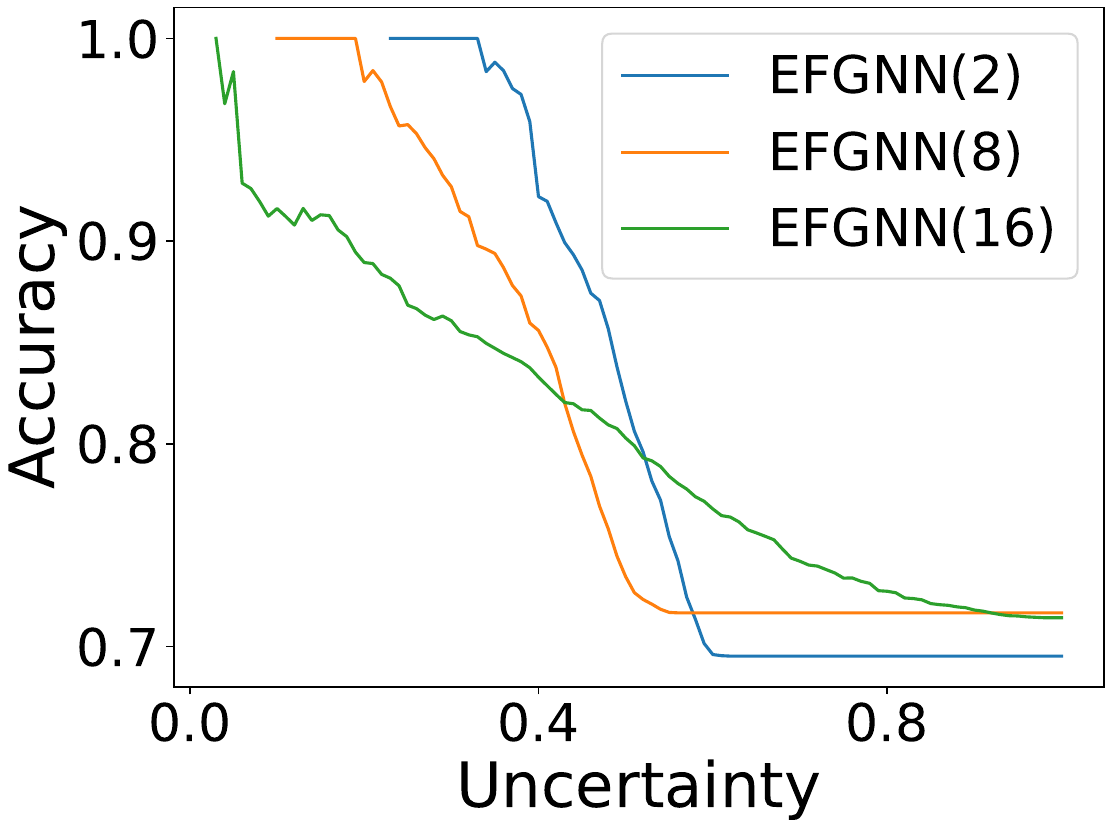}}\hfill
	\subfloat[Pubmed\label{subfig:acc_u_pubmed}]{
		\includegraphics[scale=0.17]{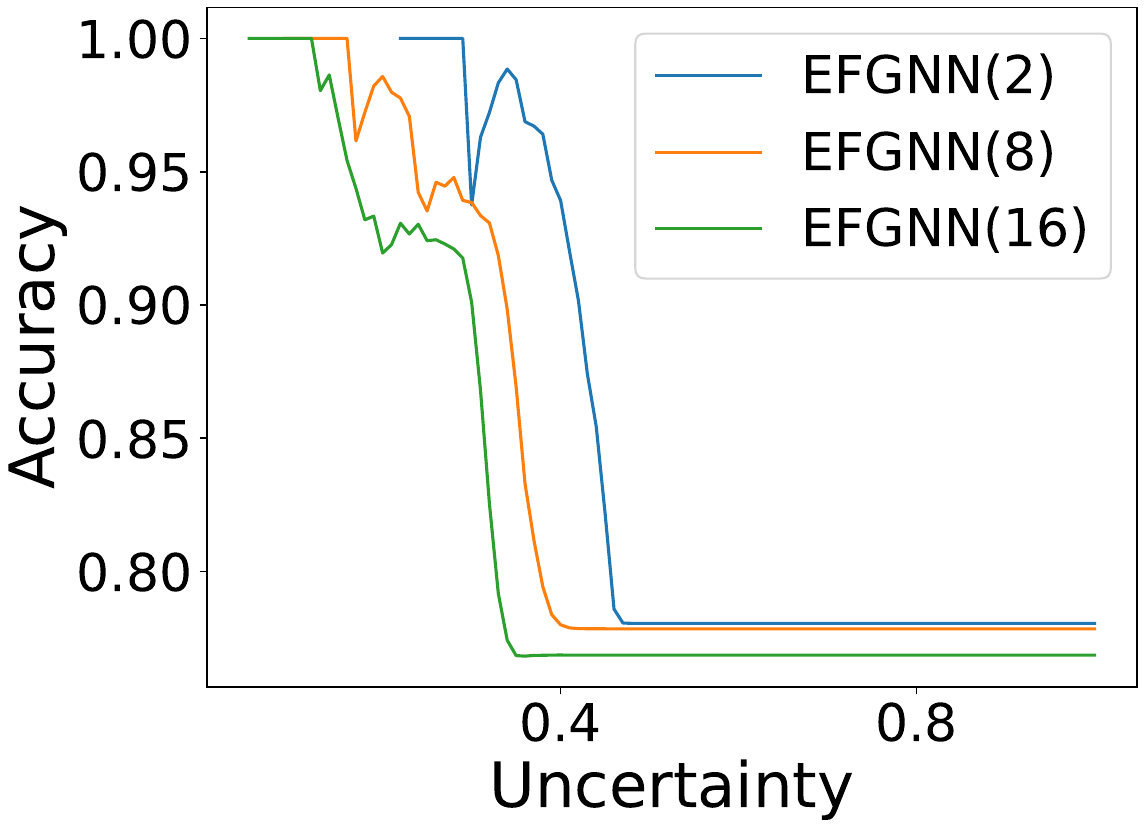}}\hfill
	\subfloat[Computers\label{subfig:acc_u_computers}]{
		\includegraphics[scale=0.17]{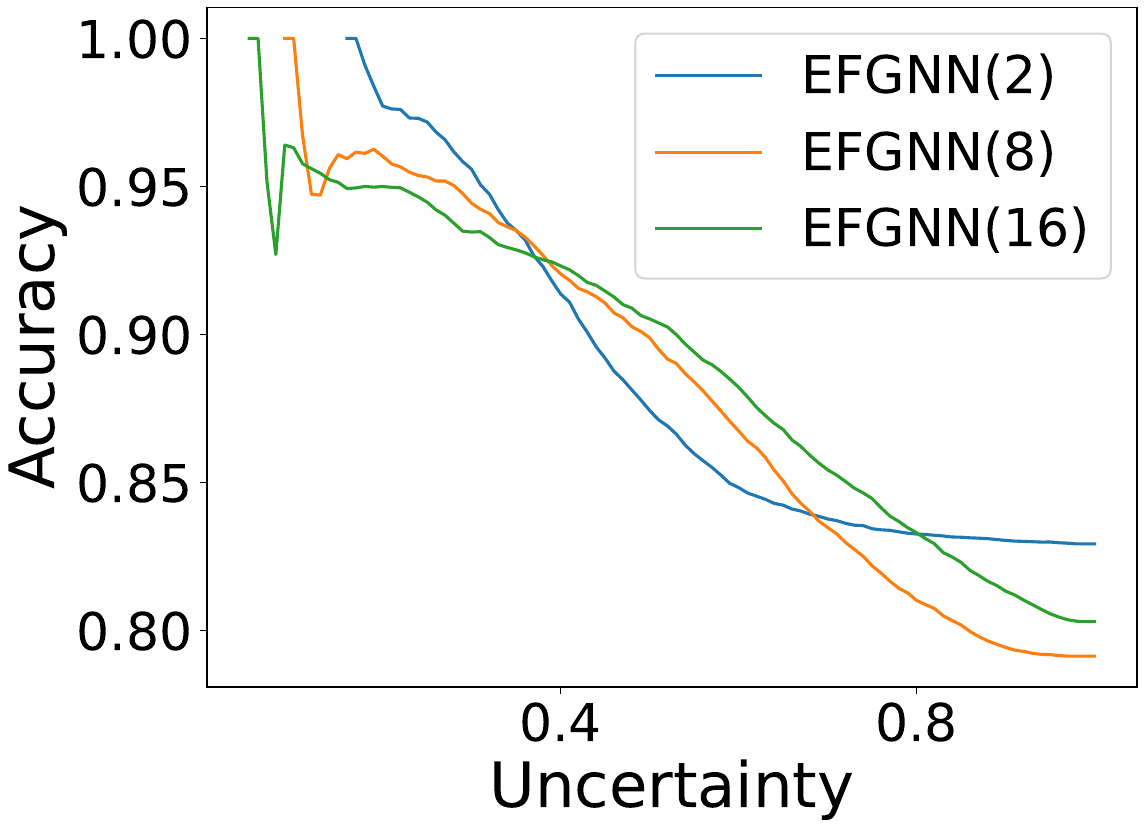}}\hfill
	\subfloat[Photo\label{subfig:acc_u_photo}]{
		\includegraphics[scale=0.17]{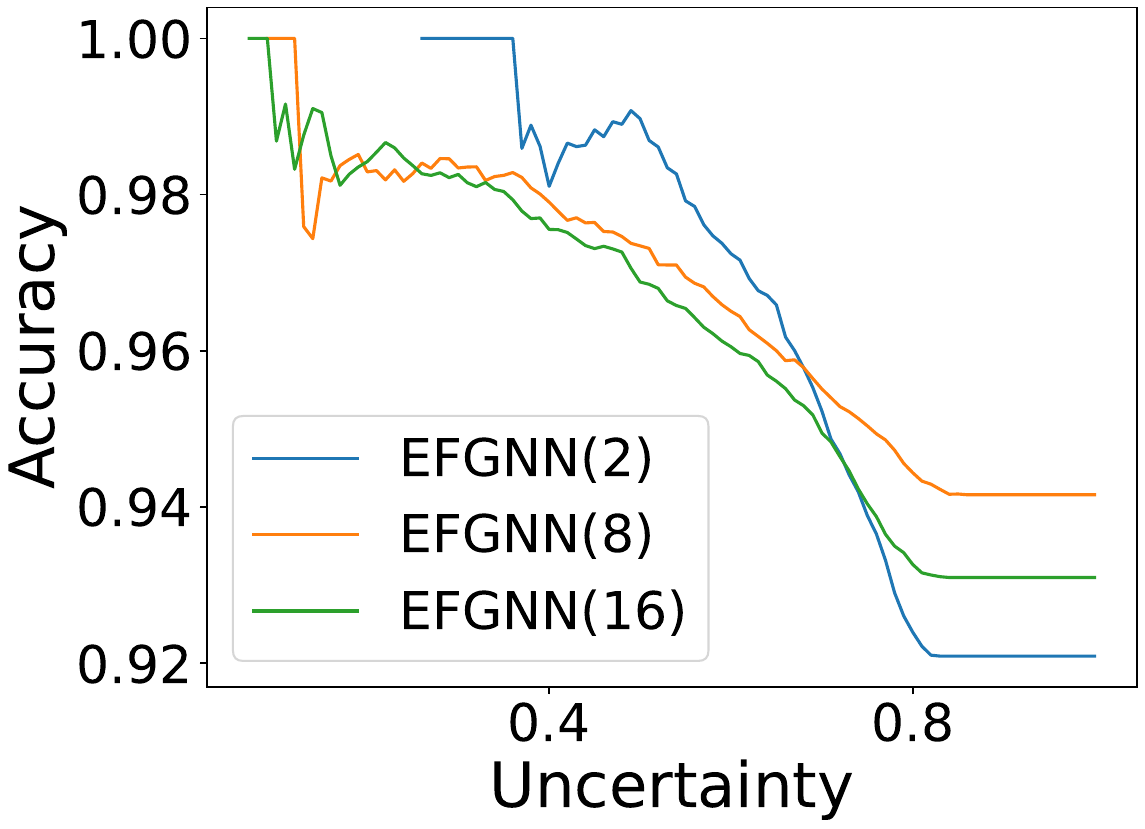}}
    \\

    \subfloat[Cora\label{subfig:noise_u_dis_cora}]{
		\includegraphics[scale=0.17]{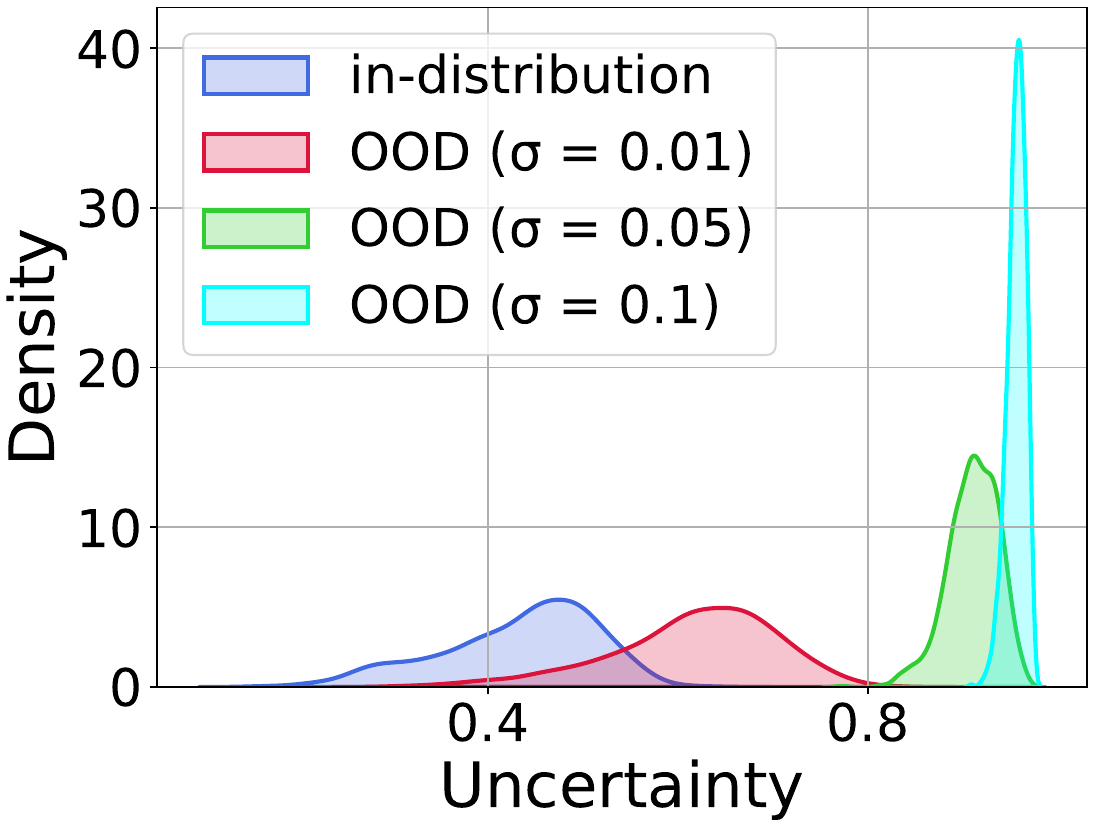}}\hfill
	\subfloat[Citeseer\label{subfig:noise_u_dis_citeseer}]{
		\includegraphics[scale=0.17]{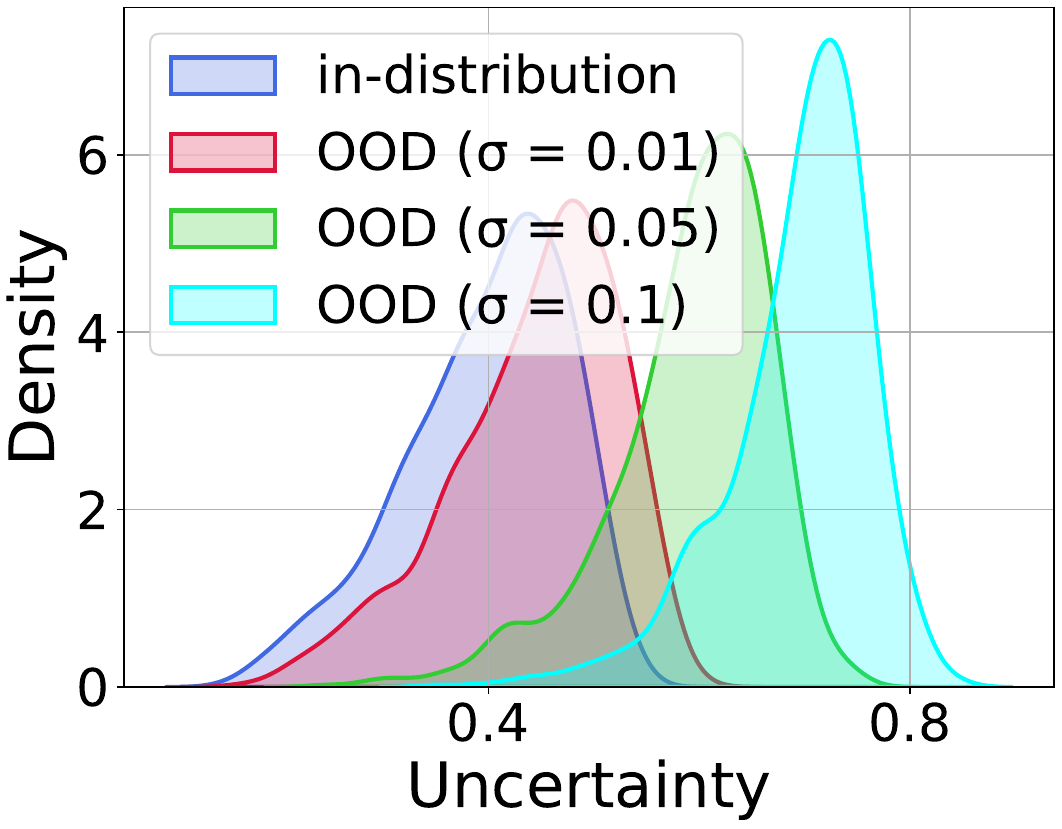}}\hfill
	\subfloat[Pubmed\label{subfig:noise_u_dis_pubmed}]{
		\includegraphics[scale=0.17]{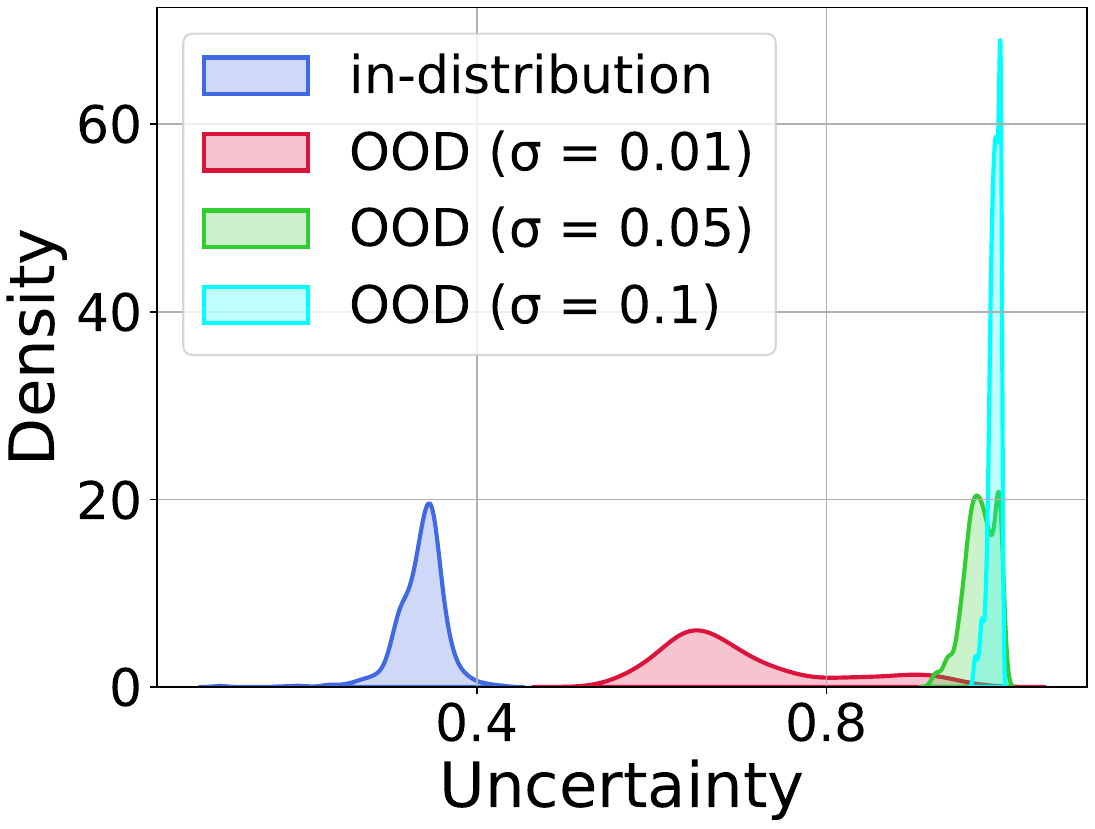}}\hfill
	\subfloat[Computers\label{subfig:noise_u_dis_computers}]{
		\includegraphics[scale=0.17]{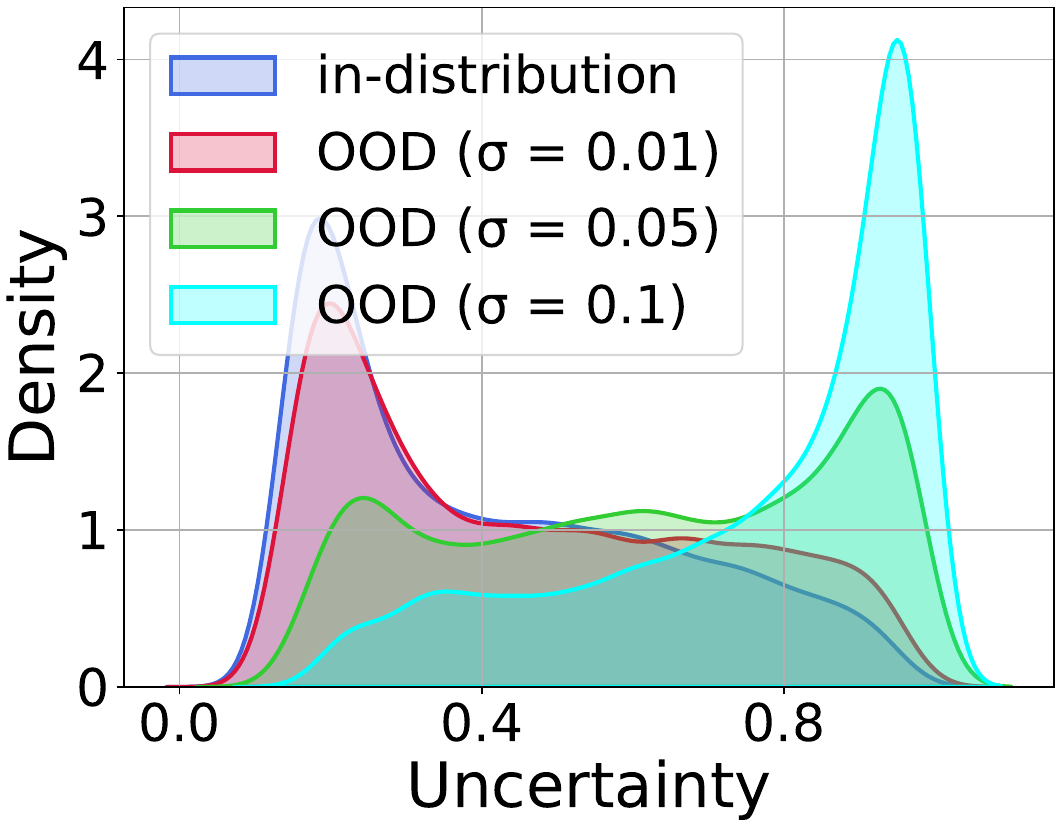}}\hfill
	\subfloat[Photo\label{subfig:noise_u_dis_photo}]{
		\includegraphics[scale=0.17]{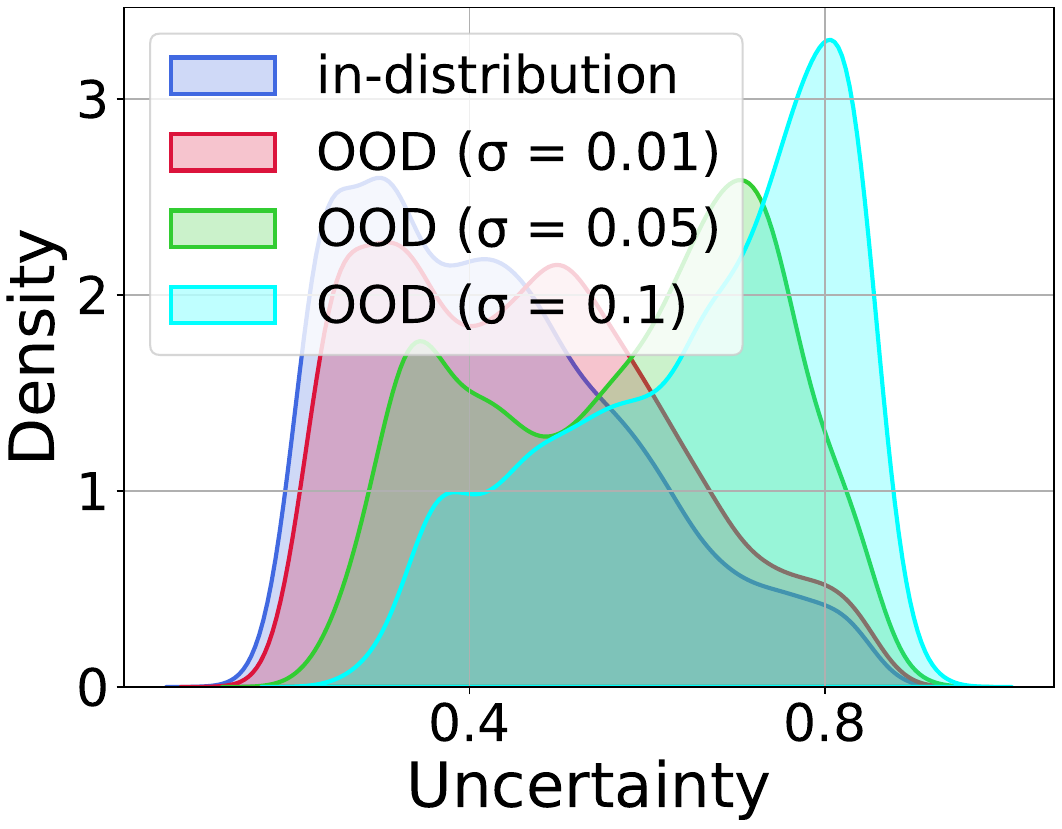}}
	\caption{Validity assessment of uncertainty on 5 datasets. (a)-(e): uncertainty threshold and accuracy; (f)-(j): uncertainty density distribution.}
	\label{fig:Uncertainty_Val}
\end{figure*}

Figure~\ref{fig:std_dis_compare} visualizes the density map of the standard deviation of the final output class probability distribution of GCN and \ourmethod. For fair comparisons, the evidence distribution of \ourmethod output is also processed by $\operatorname{SoftMax}$ when the depth is 2, 16, and 32. The results on all datasets demonstrate that the standard deviation of 2-layer GCNs is comparable to that of \ourmethod. However, with the number of layers increasing, the standard deviation of the 16-layer and 32-layer GCNs is positioned at the far left in the graph. This suggests that deep GCN models cannot make prediction with sufficient confidence. In contrast, \ourmethod has amassed substantial evidence for its prediction, granting it ample confidence through evidence fusion.

To further investigate the relative importance of evidence learned at different hops, we analyzed the average uncertainty of node predictions across various neighborhood distances. Specifically, we computed the mean uncertainty for each hop and visualized this alongside the corresponding model performance. As shown in Figure~\ref{fig:hop&ua}, we have the following observations:
the uncertainty associated with information across different hops fluctuates dynamically, and these variations are inconsistent across datasets. Uncertainty is moderately correlated with accuracy, with higher uncertainty typically indicating lower accuracy. Regardless of the variability in uncertainty across individual hops, the fused uncertainty is consistently the lowest, corresponding to the highest accuracy.
These results underscore that even high-uncertainty neighborhoods provide useful signals. Therefore, long-range propagation coupled with our CBF aggregation, which is based on evidence theory and thus fixed as a principled mechanism proves essential for effectively integrating multi-hop information to achieve optimal performance.

\subsection{Uncertainty Effectiveness Evaluation}

To validate the reliability of the uncertainty, we initially examine the impact of prediction uncertainty on model performance. To achieve this, we choose samples with uncertainties falling within certain threshold values across all node predictions and then compute the accuracy. As shown in Figure~\ref{fig:Uncertainty_Val}~(a)-(e), this reveals a consistent pattern between \ourmethod's classification accuracy and the uncertainty threshold across the five datasets. Specifically, as the uncertainty threshold decreases, the selected predictions are the most reliable ones, and the prediction accuracy of the model is improved; conversely, as the uncertainty threshold increases, the prediction accuracy diminishes. This means that specific thresholds of uncertainty can be set to enable more reliable predictions in scenarios with higher security, so the output of the model (the classification prediction and its corresponding uncertainty) is reasonable and supports trustworthy predictions.

In addition, we visually compare the uncertainty of the OOD (containing Gaussian noise) with that of the in-distribution (original) sample. Similar to work~\cite{geng2021uncertainty}, the noise vector (i.e., $\epsilon$) is sampled from the Gaussian distribution $N(0,I)$ and polluted the test set with $\eta$ intensity, i.e. $\widetilde{\mathbf{X}}=\mathbf{X}+\eta \boldsymbol{\epsilon}$. As shown in Figure~\ref{fig:Uncertainty_Val} (f)-(j), the model usually gives a higher estimate of uncertainty for the noise data outside the distribution. This shows that the model is effective in characterizing uncertainty, since it can distinguish between these data well.

\begin{figure} [h]
\centering{\includegraphics[scale=0.6]{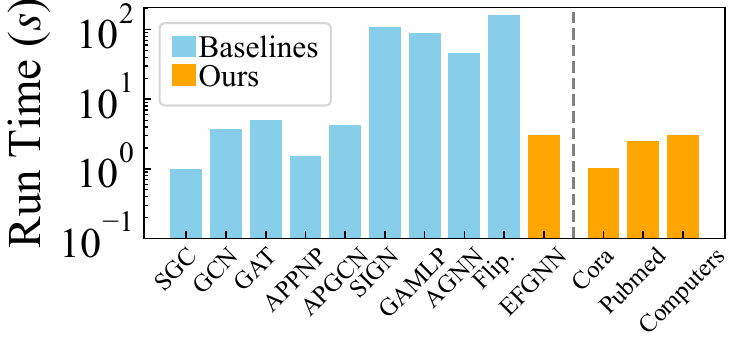}}
\caption{Efficiency comparison among various baseline methods on the Computers dataset.
And efficiency of \ourmethod across datasets of varying scales.}
\label{fig:depth}
\end{figure}

\subsection{Efficiency Analysis}
To further investigate the efficiency of \ourmethod, we conducted a study assessing the computational impact of our multi-hop evidence learning and CBF process. Specifically, by comparing EFGNN against a range of representative GNN architectures, including simpler models (SGC, GCN, GAT, APPNP) and more complex ones (SIGN, GAMLP, AGNN) by running each method for 100 epochs on the Computers dataset under identical conditions (see Figure~\ref{fig:depth}).
\textbf{1) Comparative efficiency with other GNNs.}
Despite the additional steps required for multi-hop evidence learning, the runtime of EFGNN is very competitive when compared to simpler GNN architectures. In contrast, complex models such as SIGN and Flip-APPNP incur significantly higher computational costs. This demonstrates that EFGNN strikes a favorable balance between efficiency and performance.
\textbf{2) Flexibility with dataset size.}
Our experiments also reveal that the computational cost scales moderately with the size of the dataset. For instance, the runtime increases from 1.03 seconds on the Cora dataset (2,708 nodes) to 3.04 seconds on Computers (13,381 nodes). This moderate increase confirms that the overhead introduced by our advanced multi-hop evidence learning and CBF process remains manageable even as graph size grows.

%% file: Tables/main_table.tex
\begin{table*}
\centering
\caption{Test accuracy on node classification benchmarks. $\first{First}$, \second{second}, and \third{third} indicate the top-three performance, respectively.}
\resizebox{\textwidth}{!}{
\begin{tabular}{l|cccccccc} 
\toprule
\textbf{Method} & \textbf{Cora} & \textbf{Citeseer} & \textbf{Pubmed} & \textbf{Computers} & \textbf{Photo} & \textbf{Actor} & \textbf{Chameleon} & \textbf{Squirrel} \\ 
\midrule
\textbf{GCN} & $81.8{\scriptstyle\pm0.5}$ & $70.8{\scriptstyle\pm0.5}$ & $79.3{\scriptstyle\pm0.7}$ & $82.4{\scriptstyle\pm0.4}$ & $91.2{\scriptstyle\pm0.6}$ & $23.5{\scriptstyle\pm0.5}$ & $28.2{\scriptstyle\pm1.2}$ & $23.3{\scriptstyle\pm0.3}$ \\ 
\textbf{GAT} & $83.0{\scriptstyle\pm0.7}$ & $72.5{\scriptstyle\pm0.7}$ & $79.0{\scriptstyle\pm0.3}$ & $80.1{\scriptstyle\pm0.6}$ & $90.8{\scriptstyle\pm1.0}$ & $25.7{\scriptstyle\pm0.8}$ & $24.7{\scriptstyle\pm2.4}$ & $22.6{\scriptstyle\pm0.3}$ \\ 
\textbf{JK-Net} & $81.8{\scriptstyle\pm0.5}$ & $70.7{\scriptstyle\pm0.7}$ & $78.8{\scriptstyle\pm0.7}$ & $82.0{\scriptstyle\pm0.6}$ & $91.9{\scriptstyle\pm0.7}$ & $24.9{\scriptstyle\pm0.4}$ & $27.7{\scriptstyle\pm2.0}$ & $24.6{\scriptstyle\pm0.3}$ \\ 
\textbf{ResGCN} & $82.2{\scriptstyle\pm0.6}$ & $70.8{\scriptstyle\pm0.7}$ & $78.3{\scriptstyle\pm0.6}$ & $81.1{\scriptstyle\pm0.7}$ & $91.3{\scriptstyle\pm0.9}$ & $23.8{\scriptstyle\pm2.8}$ & $21.3{\scriptstyle\pm3.9}$ & $22.4{\scriptstyle\pm0.7}$ \\ 
\textbf{APPNP} & $83.3{\scriptstyle\pm0.5}$ & $71.8{\scriptstyle\pm0.5}$ & $\third{80.1{\scriptstyle\pm0.2}}$ & $81.7{\scriptstyle\pm0.3}$ & $91.4{\scriptstyle\pm0.3}$ & $23.1{\scriptstyle\pm0.5}$ & $30.0{\scriptstyle\pm0.4}$ & $24.1{\scriptstyle\pm0.2}$ \\ 
\textbf{AP-GCN} & $83.4{\scriptstyle\pm0.3}$ & $71.3{\scriptstyle\pm0.5}$ & $79.7{\scriptstyle\pm0.3}$ & $\third{83.7{\scriptstyle\pm0.6}}$ & $92.1{\scriptstyle\pm0.3}$ & $20.7{\scriptstyle\pm2.9}$ & $38.3{\scriptstyle\pm4.4}$ & $\first{30.2{\scriptstyle\pm1.8}}$ \\ 
\textbf{SGC} & $81.0{\scriptstyle\pm0.2}$ & $71.3{\scriptstyle\pm0.5}$ & $78.9{\scriptstyle\pm0.5}$ & $82.2{\scriptstyle\pm0.9}$ & $91.6{\scriptstyle\pm0.7}$ & $22.9{\scriptstyle\pm0.2}$ & $28.7{\scriptstyle\pm0.3}$ & $24.0{\scriptstyle\pm0.2}$ \\ 
\textbf{SIGN} & $82.1{\scriptstyle\pm0.3}$ & $72.4{\scriptstyle\pm0.8}$ & $79.5{\scriptstyle\pm0.5}$ & $83.1{\scriptstyle\pm0.8}$ & $91.7{\scriptstyle\pm0.7}$ & $\third{28.5{\scriptstyle\pm0.1}}$ & $34.9{\scriptstyle\pm0.4}$ & $24.2{\scriptstyle\pm0.2}$ \\ 
\textbf{S\textsuperscript{2}GC} & $82.7{\scriptstyle\pm0.3}$ & $73.0{\scriptstyle\pm0.2}$ & $79.9{\scriptstyle\pm0.3}$ & $83.1{\scriptstyle\pm0.7}$ & $91.6{\scriptstyle\pm0.6}$ & $23.5{\scriptstyle\pm0.2}$ & $28.7{\scriptstyle\pm0.4}$ & $24.0{\scriptstyle\pm0.3}$ \\ 
\textbf{GAMLP} & $\second{83.9{\scriptstyle\pm0.6}}$ & $\first{73.9{\scriptstyle\pm0.6}}$ & $\second{80.8{\scriptstyle\pm0.5}}$ & $\second{84.2{\scriptstyle\pm0.5}}$ & $\second{92.6{\scriptstyle\pm0.8}}$ & $22.2{\scriptstyle\pm1.4}$ & $\third{44.1{\scriptstyle\pm1.1}}$ & $\third{29.3{\scriptstyle\pm2.6}}$ \\ 
\textbf{AGNN} & $80.2{\scriptstyle\pm0.6}$ & $70.4{\scriptstyle\pm0.5}$ & $77.1{\scriptstyle\pm0.5}$ & $75.7{\scriptstyle\pm2.4}$ & $88.9{\scriptstyle\pm1.5}$ & $24.0{\scriptstyle\pm3.0}$ & $\first{45.0{\scriptstyle\pm1.4}}$ & $27.6{\scriptstyle\pm0.4}$ \\ 
\textbf{Flip-APPNP} & $81.5{\scriptstyle\pm0.6}$ & $70.4{\scriptstyle\pm0.5}$ & $79.8{\scriptstyle\pm0.6}$ & $80.8{\scriptstyle\pm0.1}$ & $91.9{\scriptstyle\pm0.3}$ & $\second{29.9{\scriptstyle\pm1.1}}$ & $26.2{\scriptstyle\pm1.8}$ & $23.6{\scriptstyle\pm0.6}$ \\ 
\midrule
\textbf{Drop-GCN} & $82.8{\scriptstyle\pm0.8}$ & $72.6{\scriptstyle\pm0.7}$ & $79.0{\scriptstyle\pm0.3}$ & $81.2{\scriptstyle\pm0.5}$ & $92.0{\scriptstyle\pm0.1}$ & $20.5{\scriptstyle\pm0.6}$ & $31.1{\scriptstyle\pm1.0}$ & $22.8{\scriptstyle\pm0.2}$ \\ 
\textbf{S-BGCN-T} & $78.3{\scriptstyle\pm1.5}$ & $\third{73.2{\scriptstyle\pm0.5}}$ & $79.1{\scriptstyle\pm0.2}$ & $78.3{\scriptstyle\pm1.5}$ & $81.8{\scriptstyle\pm1.3}$  & $21.6{\scriptstyle\pm0.6}$ & $31.4{\scriptstyle\pm0.4}$ & $22.4{\scriptstyle\pm1.1}$ \\ 
\textbf{GKDE} & $\third{83.8{\scriptstyle\pm0.7}}$ & $73.0{\scriptstyle\pm0.7}$ & $79.1{\scriptstyle\pm0.2}$ & $78.1{\scriptstyle\pm2.0}$ & $88.3{\scriptstyle\pm1.6}$  & $19.7{\scriptstyle\pm1.1}$ & $31.5{\scriptstyle\pm0.3}$ & $23.3{\scriptstyle\pm0.2}$ \\ 
\textbf{CaGCN} & $83.4{\scriptstyle\pm0.8}$ & $72.8{\scriptstyle\pm1.1}$ & $79.3{\scriptstyle\pm1.5}$ & $82.0{\scriptstyle\pm1.6}$ & $\third{92.3{\scriptstyle\pm0.6}}$  & $25.7{\scriptstyle\pm1.0}$ & $29.3{\scriptstyle\pm0.4}$ & $23.0{\scriptstyle\pm2.0}$ \\ 
\midrule
\textbf{\ourmethod} & $\first{85.2{\scriptstyle\pm0.3}}$ & $\second{73.3{\scriptstyle\pm0.4}}$ & $\first{82.4{\scriptstyle\pm0.2}}$ & $\first{84.5{\scriptstyle\pm0.4}}$ & $\first{93.8{\scriptstyle\pm0.1}}$  & $\first{30.0{\scriptstyle\pm0.3}}$ & $\second{44.3{\scriptstyle\pm0.8}}$ & $\second{29.5{\scriptstyle\pm1.2}}$ \\
\bottomrule
\end{tabular}
}
\label{t2}
\end{table*}

%% file: Tables/ablation.tex
\begin{table}
\centering
\caption{Classification accuracy of \ourmethod and its variant. $\first{First}$, \second{second}, and \third{third} indicate the top-three performance, respectively.}
\resizebox{1\columnwidth}{!}{
\begin{tabular}{l|ccccc} 
\toprule
\textbf{Method} & \textbf{Cora} & \textbf{Citeseer} & \textbf{Pubmed} & \textbf{Computers} & \textbf{Photo} \\
\midrule
\textbf{w/o KL} & $\second{84.1}$ & $\second{71.7}$ & $\second{82.1}$ & $\second{80.9}$ & $\second{93.6}$ \\
\textbf{w/o Dis} & $\third{83.8}$ & 70.5 & $\third{81.7}$ & $\third{80.3}$ & $\third{93.1}$ \\
\textbf{w/o CBF} & 82.0 & $\third{71.3}$ & 80.5 & 78.3 & 87.7 \\
\midrule
\textbf{\ourmethod} & $\first{85.2}$ & $\first{73.3}$ & $\first{82.4}$ & $\first{84.5}$ & $\first{93.8}$ \\
\bottomrule
\end{tabular}
}
\label{t3}
\end{table}

\begin{table}
\centering
\caption{Classification accuracy of \ourmethod and the variants using each layer output for final prediction. $\first{First}$, \second{second}, and \third{third} indicate the top-three performance, respectively.}
\resizebox{1\columnwidth}{!}{
\begin{tabular}{l|ccccc} 
\toprule
\textbf{Method} & \textbf{Cora} & \textbf{Citeseer} & \textbf{Pubmed} & \textbf{Computers} & \textbf{Photo} \\ 
\midrule
\textbf{EP-0} & 71.9 & 64.0 & 76.5 & 77.1 & 88.2 \\ 
\textbf{EP-1} & 81.1 & 69.5 & 79.4 & $\third{83.9}$ & $\third{92.3}$ \\ 
\textbf{EP-2} & $\second{83.3}$ & $\second{72.5}$ & $\third{81.0}$ & $\second{84.0}$ & $\second{92.9}$ \\ 
\textbf{EP-3} & $\third{83.1}$ & $\third{72.3}$ & 80.5 & 83.3 & 92.1 \\ 
\textbf{EP-4} & 82.4 & 71.8 & 80.9 & 82.8 & 91.5 \\ 
\textbf{EP-5} & 82.5 & 71.7 & $\second{81.1}$ & 82.3 & 90.6 \\ 
\textbf{EP-6} & 82.3 & 71.6 & $\second{81.1}$ & 81.3 & 89.7 \\ 
\textbf{EP-7} & 82.4 & 71.4 & 80.7 & 80.2 & 88.8 \\ 
\textbf{EP-8} & 82.2 & 71.3 & 80.5 & 79.4 & 87.7 \\ 
\midrule
\textbf{\ourmethod} & $\first{85.2}$ & $\first{73.3}$ & $\first{82.4}$ & $\first{84.5}$ & $\first{93.8}$ \\
\bottomrule
\end{tabular}
}
\label{t4}
\end{table}

%% file: Tables/depth.tex
\begin{table}
\centering
\caption{Classification accuracy of GCN / \ourmethod with different depths. $\first{First}$, \second{second}, and \third{third} indicate the top-three performance, respectively.}
\resizebox{1\columnwidth}{!}{
\begin{tabular}{l|lllll} 
\toprule
\textbf{Dataset} & \multicolumn{1}{c}{\begin{tabular}[c]{@{}c@{}}\textbf{2}\\\end{tabular}} & \multicolumn{1}{c}{\begin{tabular}[c]{@{}c@{}}\textbf{8}\\\end{tabular}} & \multicolumn{1}{c}{\begin{tabular}[c]{@{}c@{}}\textbf{16}\\\end{tabular}} & \multicolumn{1}{c}{\begin{tabular}[c]{@{}c@{}}\textbf{32}\\\end{tabular}} & \multicolumn{1}{c}{\begin{tabular}[c]{@{}c@{}}\textbf{64}\\\end{tabular}} \\ 
\midrule
\textbf{Cora} & 81.8/80.9 & 69.5/85.2 & 69.4/$\third{85.3}$ & 60.3/$\first{85.9}$ & 28.7/$\second{85.4}$ \\ 
\textbf{Citeseer} & 70.8/71.5 & 30.2/$\second{73.3}$ & 18.3/$\first{74.5}$ & 25.0/$\third{73.0}$ & 20.0/72.8 \\ 
\textbf{Pubmed} & 79.3/79.2 & 61.2/$\third{82.4}$ & 40.9/$\first{83.0}$ & 22.4/$\second{82.5}$ & 35.3/82.2 \\ 
\textbf{Computers} & 82.4/$\second{84.3}$ & 67.7/$\first{84.5}$ & 28.3/$\third{83.2}$ & 17.0/83.0 & 13.6/82.8 \\ 
\textbf{Photo} & 91.2/$\third{93.2}$ & 82.4/$\first{93.8}$ & 69.2/$\second{93.5}$ & 38.3/$\third{93.2}$ & 17.1/$\third{93.2}$ \\
\bottomrule
\end{tabular}
}
\label{t5}
\end{table}

%% file: 7_Conclusion.tex
In this paper, we propose a robust and trustworthy graph neural network model, \ourmethod, based on uncertainty-aware for graph semi-supervised learning. \ourmethod focuses on fusing the evidence from different hops of neighborhoods to achieve reliable prediction. Our proposed model can effectively identify the high-risk prediction and use the evidence fusion mechanism to generate the most confident prediction. In addition, the model can reflect the reliability of the current prediction in the final classification, making possible failure prediction transparent. Both theoretical analysis and experimental results confirm the effectiveness of the \ourmethod in prediction accuracy and uncertainty estimation.

Furthermore, link prediction and graph-level classification also encounter challenges stemming from uncertain predictions. Incorporating evidence-based learning could significantly boost both the reliability and credibility of models in these domains. In future work, we plan to extend our framework to these tasks by refining the evidence fusion strategy, a refinement we believe may broaden the approach's applicability and yield deeper insights into robust graph representation learning.